\documentclass{article} 
\usepackage{iclr2026_conference,times}

\usepackage{amsmath,amsfonts,bm}

\def\eqref#1{equation~\ref{#1}}

\def\1{\bm{1}}

\DeclareMathAlphabet{\mathsfit}{\encodingdefault}{\sfdefault}{m}{sl}
\SetMathAlphabet{\mathsfit}{bold}{\encodingdefault}{\sfdefault}{bx}{n}

\usepackage{hyperref}
\usepackage{url}

\usepackage{algorithm}
\usepackage{algorithmic}

\usepackage{graphicx}
\usepackage{booktabs}
\usepackage{multirow}
\usepackage[table]{xcolor}
\usepackage{enumitem}
\usepackage{amsmath,amssymb,amsthm}
\usepackage{natbib}
\usepackage{algorithm}
\usepackage{algorithmic}
\newtheorem{theorem}{Theorem}

\theoremstyle{definition}

\theoremstyle{remark}

\usepackage{soul}
\usepackage{xcolor}

\usepackage{xspace}
\newcommand{\shortname}{TESS\xspace}
\newcommand{\longname}{Transferable Example Scoring and Selection\xspace}

\usepackage{titletoc}
\titlecontents{section}
[1.5pc]                              
{\addvspace{6pt}\bfseries}           
{\contentslabel{1.5pc}}           
{}                                  
{\titlerule*[0.5pc]{.}\contentspage}
\usepackage{wrapfig}

\title{Scalable, Transferable Meta-network for Data Selection Requires a Different Loss (and Why the Obvious Choice is Problematic)}

\author{Zilin Du\\
College of Computing and Data Science \\
Nanyang Technological University \\
Singapore \\
\texttt{zilin003@e.ntu.edu.sg} \\
\And
Bowen Yang \\
College of Computing and Data Science \\
Nanyang Technological University \\
Singapore \\
\texttt{yang1117@e.ntu.edu.sg} \\
\And
Boyang Albert Li  \thanks{ Corresponding author.} \\
College of Computing and Data Science \\
Nanyang Technological University \\
Singapore \\
\texttt{boyang.li@ntu.edu.sg}
}

\iclrfinalcopy 
\begin{document}

\maketitle

\begin{abstract}

Data selection is critical for training large language models on
massive and heterogeneous corpora. Meta-learning for Training-data Selection offers a principled alternative to heuristic scoring by learning data
weights from a target validation objective, but existing methods face a
trade-off between fine-grained valuation and transferability to unseen data.
A natural solution is to replace per-sample weights with a selection network.
However, we find that directly incorporating such a network into existing MTS
objectives leads to unstable optimization and poor generalization, caused by
weight suppression and persistent reliance on easy-to-learn features. To
address these issues, we propose \longname (\shortname), a scalable data-selection
framework built on a Pointwise Value Matching objective (PVM). Experiments on LLM safety and targeted instruction tuning demonstrate
strong transfer across datasets, from subsets to full corpora, and from
smaller to larger models.

\end{abstract}

\vspace{-0.5em}
\section{Introduction}
\vspace{-0.5em}
The performance of modern AI systems is fundamentally shaped by the training data \citep{zha2025data,kumar2024opportunities,zha2023data,whang2023data}. For large language models (LLMs), scaling requires assembling massive
corpora from heterogeneous sources \citep{zhou2025survey,weber2024redpajama,langlais2026common,li2024datacomp}, which often introduces redundant and
low-quality examples. Training data selection is hence an important component of LLM training.
Many existing methods rank training examples by manually designed scores, such as entropy \citep{li2026unified}, instruction-following difficulty \citep{li2024quantity}, reducible training loss \citep{mindermann2022prioritized}, and distance to target domain \citep{deng2026less, min2026gist}, to name a few. These scoring rules lack the ability to dynamically adapt to new scenarios. Though each captures a potentially useful signal, it remains an open question how these scores may be combined to create more powerful ranking mechanisms.

Meta-learning for Training-data Selection (MTS) \citep{shen2025seal, pan2025scalebio, ren2018learning, wang2020meta, du2026difficulty, shu2019meta, yu2025llm,calian2026datarater}
provides a principled alternative, where the values of training data are learned directly from their effects on the validation
loss $\mathcal{L}_{\mathrm{val}}$. Let $\bm{\theta}$ denote the LLM parameters, $w_i > 0$ be the learnable weight for the loss of the $i$-th training data point, $\ell_i(\bm{\theta})$, and the vector $\bm w = (w_i)_i$ contain all $w_i$ as components.
MTS formulates data weighting as the following
bilevel optimization problem:
\begin{equation}
\label{eq:bi-level-problem}
    \mathbf{w}^*
    =
    \arg\min_{\mathbf{w}}
    \mathcal{L}_{\mathrm{val}}
    \bigl(\bm\theta^*(\mathbf{w})\bigr), \qquad
    \text{s. t.}\quad
    \bm\theta^*(\mathbf{w})
    =
    \arg\min_{\bm\theta}
    \|\bm{w}\|_1^{-1} \sum_i w_i \ell_i(\bm{\theta}).
\end{equation}
This framework allows automatic identification of effective training data, which lead the model to generalize to the validation data.

Nevertheless, MTS in practice runs into two obstacles; the first is scalability. Computing the exact hypergradient
$\nabla_{\bm w}\mathcal{L}_{\mathrm{val}}(\bm\theta^*(\mathbf{w}))$
by unrolling the inner optimization is costly.
Even unrolling one step and taking the finite-different approximation require three backward operations on $\bm{\theta}$ for one gradient step on $\bm{w}$ \citep{du2026difficulty,shu2019meta}, which is prohibitive expensive for large LLMs.  \citet{calian2026datarater} manage to unroll 8 steps using 64 v5 TPUs, which is beyond the capability of most academic labs. 
With penalty-based reformulations~\citep{shen2023penalty, xiao2023generalized, shen2025principled}, ScaleBiO~\citep{pan2025scalebio} reduces the number of backward calls to two, whereas SEAL \citep{shen2025seal} reduces it further to one, but incurs discernible performance losses. 

The second challenge is generalization to unseen data. In the formulation of \citet{ren2018learning,shen2025seal}, every data point receives a separate weight $w_i$, and no generalization is possible. An alternative approach~\citep{pan2025scalebio} is to assign a separate weight for each data domain (e.g., one dataset in a collection of multiple datasets), which provides generalization within each domain. However, this overlooks heterogeneity among data points in a single domain and cannot generalize to new data domains. 

One intuitively appealing idea is to train a meta-network that outputs $w_i$ for every data point. If the meta-network can  generalize to unseen data points, we would be able to train on a small dataset and perform data selection on a much larger dataset, which would also substantially improve efficiency and scalability.  Though the meta-network approach has been used by \citet{shu2019meta,du2026difficulty,calian2026datarater}, no such generalization has been reported, not even for the extraordinarily large-scale experiments of \citet{calian2026datarater}. Perhaps unsurprisingly, our experiments find simplistic training of the meta-network leads to unstable optimization (Figure \ref{fig:weight_dynamics}) and poor generalization (Table \ref{tab:shortcut_generalization} and \ref{tab:ablation_selector_loss}).



Given the massive success of deep learning, the failure of the meta-network to generalize to unseen data is especially perplexing. Our analysis traces the 
failure to two properties of the training loss. First, as training progresses, most predicted data weights are suppressed toward zero. As a result, over time, the meta-network loses the ability to distinguish useful training data from harmful training data. Second, the loss causes the network to exclusively learn easy features and ignores hard feature, so the network becomes susceptible to shortcut learning.


To improve the scalability and generalization of MTS, we propose \longname (\shortname). At its core, \shortname simplifies  ScaleBiO \citep{pan2025scalebio} and introduces a simple mean squared error objective, which we call Pointwise Value Matching (PVM). \shortname{} constructs sample-level pseudo-labels that capture
the utility of training data points and uses these values as supervision for the data-selecting meta-network. Compared to the existing loss, the PVM loss is less focused on easy features, which curbs shortcut learning and promotes generalization. 
Further, \shortname{} removes the requirement of ScaleBiO to simultaneously train two LLMs with the meta-network, cutting peak memory footprint by more than 20\%.  
Experiments on LLM safety and targeted instruction tuning demonstrate substantially improved generalization of \shortname{} across three
settings: cross-dataset transfer,
small-to-large model transfer, and subset-to-full-corpus transfer. Our contributions can be summarized as:
\begin{itemize}[leftmargin=*]
\item We identify and theoretically analyze two limitations of applying the existing MTS training objective to a shared selection network: competitive weight
suppression that undermines the identification of high-quality training data, and persistent
preference for easy-to-learn features that can encourage shortcut learning.
\item We propose \shortname, a transferable data-selection
framework that constructs sample-level pseudo-labels and uses them as supervision to learn a data selection network generalizable to
unseen data. To the best of our knowledge, this is the first work that demonstrates generalization to unseen data in MTS. 
\end{itemize}

\vspace{-0.5em}
\section{Related Work}
\vspace{-0.5em}
\textbf{Meta-Learning Data Selection for LLMs.}
Scaling meta-learning data selection to LLMs requires efficiency in face of both huge parameter counts~\citep{pan2025scalebio} and dataset sizes~\citep{xie2023data}. Traditional data selection methods~\citep{ren2018learning,shu2019meta} differentiate through model updates to compute the hypergradient, which incurs large time and space complexity for large models. ScaleBiO \citep{pan2025scalebio} and SEAL~\citep{shen2025seal} focus on scaling to model parameter counts. If data selection can be trained on a small dataset and generalize to unseen data, it would also significantly improve scalability to large datasets. However, such generalization remains a difficult challenge. 

In terms of generalization, existing approaches can be broadly grouped into sample-level, domain-level, and meta-network methods. Example-level methods
~\citep{xie2023data,xia2024less,shen2025seal,calian2026datarater} 
learn one weight for each training data point, thereby offering no generalization to new data points. Domain-level methods
~\citep{xie2023doremi, xie2025chameleon,pan2025scalebio,yang2026data,wen2026mixatlas}
partition the training data into domains and assign one data weight to each domain. Chameleon~\citep{xie2025chameleon} and Data Mixing Agent~\citep{yang2026data} further introduce a
shared domain taxonomy that enables limited weight sharing among data domains. However, domain-level weights are too coarse to capture the
crucial differences in data quality and content within each domain \citep{lee2024conceptskill}. 
Meta-network methods \citep{shu2019meta,du2026difficulty,calian2026datarater} train a meta-network to output the weight for each data point directly. In theory, this should allow generalization to new data points; in practice, no generalization has been reported. Notably, DataRater \citep{calian2026datarater} trains a separate scorer for each dataset,
suggesting generalization across dataset is challenging even when the network are trained with a huge amount of data.
With this paper, we analyze the reasons behind the challenge mathematically and offer a simple, practical, and high-performance solution.

\textbf{Shortcut Learning.} Shortcut learning occurs when models rely on easy features that are
strongly predictive of training labels but whose correlations
weaken or reverse under distribution shifts, impairing
out-of-distribution generalization~\citep{geirhos2020shortcut, shah2020pitfalls}. Prior work distinguishes \emph{predictivity}, how reliably a
feature predicts training labels, from \emph{availability},
how easily it can be extracted~\citep{hermann2024foundations}.
\citet{scimeca2021shortcut} find that solutions favoring easy-to-learn
cues tend to lie in flatter minima and are more abundant in
parameter space, while \citet{qiu2024complexity} show that simpler
or more strongly correlated spurious features can slow
core-feature learning. We investigate these dynamics in data selection networks. We show that the ScaleBiO loss continues to reinforce easy-to-learn features that are susceptible to shortcut learning, while our proposed PVM promotes the learning of harder and more transferable features.

\vspace{-0.5em}
\section{Methodology}
\label{sec:method}
\vspace{-0.5em}
\subsection{Problem Formulation}
\label{sec:problem}

Let $\mathcal D_{\mathrm{train}}$ and $\mathcal D_{\mathrm{val}}$ denote the
training and target validation sets, with sizes $N_{\mathrm{train}}$ and
$N_{\mathrm{val}}$, respectively. We further use $\mathcal D_{\mathrm{pool}}$ to denote the candidate pool to
be scored. For example, $\mathcal D_{\mathrm{pool}}$ can be a larger corpus
containing $\mathcal D_{\mathrm{train}}$,
$\mathcal D_{\mathrm{train}}\subsetneq\mathcal D_{\mathrm{pool}}$,
or a disjoint dataset,
$\mathcal D_{\mathrm{train}}\cap\mathcal D_{\mathrm{pool}}=\varnothing$. Let $\ell_i(\bm\theta)$ denote the
token-averaged cross-entropy loss for the $i$-th data point under an LLM parameterized by $\bm\theta$. For the training examples, let $w_i>0$ denote the weight of
the $i$-th data point, $\bm w=\smash{(w_i)_{i=1}^{N_{\mathrm{train}}}}$
the weight vector, and
$S=\|\bm w\|_1$ the total weight.

We aim to learn a data selection meta-network $s_{\bm\phi}$ that
predicts sample weights and can score unseen examples in
$\mathcal D_{\mathrm{pool}}$.
We parameterize $w_i=s_{\bm\phi}(x_i)>0$ and denote the resulting
weight vector by $\bm w(\bm{\phi})$.
In MTS, learning $\bm\phi$ is formulated as bilevel optimization:
\begin{align}
    \bm{\phi}^{*}
    &=
    \arg\min_{\bm\phi}
    \mathcal{L}\bigl(
        \bm\theta^{*};
        \mathcal{D}_{\mathrm{val}}
    \bigr)
    =
    \arg\min_{\bm\phi}
    {N_{\mathrm{val}}^{-1}}
    \sum\nolimits_{i=1}^{N_{\mathrm{val}}}
    \ell_i(\bm\theta^{*}),
    \label{eq:mts_outer}
    \\
    \text{s.t.}\quad
    \bm\theta^{*}
    &=
    \arg\min_{\bm\theta}
    \mathcal{L}\bigl(
        \bm\theta,\bm w(\bm{\phi});
        \mathcal{D}_{\mathrm{train}}
    \bigr)
    =
    \arg\min_{\bm\theta}
    \|\bm w\|_1^{-1}
    \sum\nolimits_{i=1}^{N_{\mathrm{train}}}
    w_i \, \ell_i(\bm\theta).
    \label{eq:mts_inner}
\end{align} Here,
$\mathcal{L}(\bm\theta;\mathcal{D})$ denotes the average
per-sample loss, while
$\mathcal{L}(\bm\theta,\bm w(\bm{\phi});\mathcal{D})$ denotes the loss weighted by the meta-network $\bm\phi$. Under this formulation, we optimize $\bm\phi$ such that training the LLM with its predicted weights minimizes the loss on $\mathcal{D}_{\mathrm{val}}$.

\subsection{Preliminaries on ScaleBiO}
\label{sec:scalebio}

ScaleBiO~\citep{pan2025scalebio} provides an efficient framework
for MTS at the LLM scale, optimizing data weights $\bm w$ for each data point. Let $\bm\theta_W$ denote the LLM parameters and $\bm\theta_U$ denote the parameters of an auxiliary LLM. Optimizing $\bm\theta_U$ on the training set yields $\bm\theta_U^{*} = \arg\min_{\bm\theta_U}
\mathcal{L}(\bm\theta_U,\bm w;\mathcal D_{\mathrm{train}})$. The original constraint in the bi-level problem (Eq. \ref{eq:bi-level-problem}) can then rewritten as $\bm\theta_W = \bm\theta_U^{*}$.

Relaxing the constraint to a necessary condition that $\bm\theta_W$ and $\bm\theta_U^{*}$ have the same training loss, we obtain a new constraint: $\mathcal L(\bm\theta_W,\bm w;\mathcal D_{\mathrm{train}})
    - \smash{\min_{\bm\theta_U}}
    \mathcal L
    (\bm\theta_U,\bm w;\mathcal D_{\mathrm{train}})
    =0.$
Building the Lagrangian function with the penalty coefficient $\alpha>0$, we derive a new minimax problem:
\begin{equation}
    \smash{\min_{\bm w,\bm\theta_W}
    \max_{\bm\theta_U}}\quad
    \mathcal L
    (\bm\theta_W;\mathcal D_{\mathrm{val}})
    +
    \alpha
    \Big[
        \mathcal L
        (\bm\theta_W,\bm w;\mathcal D_{\mathrm{train}})
        -
        \mathcal L
        (\bm\theta_U,\bm w;\mathcal D_{\mathrm{train}})
    \Big].
    \label{eq:scalebio_minimax}
\end{equation}
To solve Eq.\ref{eq:scalebio_minimax}, ScaleBiO simultaneously trains the two LLMs and data weights, resulting in substantial memory overhead.
SEAL~\citep{shen2025seal} additionally eliminates the auxiliary LLM $\bm\theta_U$, but this comes at an accuracy penalty. 

We write the ScaleBiO loss for data weights $\bm{w}$ explicitly:
\begin{equation}
    \smash{\mathcal L_w^{\mathrm{SBO}}
    (\bm w;\bm\theta_W,\bm\theta_U)}
    =
    \smash{ \|\bm w\|_1^{-1}
    \sum\nolimits_{i=1}^{N_{\mathrm{train}}}}
    w_i
    \big[
        \ell_i(\bm\theta_W)-\ell_i(\bm\theta_U)
    \big].
    \label{eq:sbo_loss_w}
\end{equation}
It is evident that, to optimize $L_w^{\mathrm{SBO}}$, we should assign high $w_i$ to data points with a low loss on $\bm\theta_w$ and a high loss on $\bm\theta_U$. Note that $\bm\theta_w$ is trained on both the training set $D_{\mathrm{train}}$ and the validation set $D_{\mathrm{val}}$ but $\bm\theta_U$ is trained on $D_{\mathrm{train}}$ only. Therefore, ScaleBiO selects data points that benefit from additional training on $D_{\mathrm{val}}$, which are likely more similar to $D_{\mathrm{val}}$ than $D_{\mathrm{train}}$.

\subsection{\longname (\shortname)}
\label{sec:TESS}

\begin{algorithm}[t]
\caption{Overview of the \shortname Framework}
\label{alg:ours}
\begin{algorithmic}[1]

\REQUIRE Training set $\mathcal{D}_{\mathrm{train}}$,
validation set $\mathcal{D}_{\mathrm{val}}$,
candidate pool $\mathcal{D}_{\mathrm{pool}}$,
coefficient $\alpha$

\STATE Obtain $\bm\theta_U^{*}$ by minimizing
$\mathcal{L}(\bm\theta_U;\mathcal{D}_{\mathrm{train}})$

\STATE Obtain $\bm\theta_W^{*}$ by minimizing
$\mathcal{L}(\bm\theta_W;\mathcal{D}_{\mathrm{val}})
+\alpha\mathcal{L}(\bm\theta_W;\mathcal{D}_{\mathrm{train}})$

\STATE For each $x_i\in\mathcal{D}_{\mathrm{train}}$, compute the pseudo-label
$\Delta_i(\bm\theta_U^{*},\bm\theta_W^{*})
\gets
\ell(\bm\theta_U^{*};x_i)
-
\ell(\bm\theta_W^{*};x_i)$

\STATE Train the selection network by optimizing
$\bm{\phi}^{*}
\gets
\arg\min_{\bm\phi}
\smash{\mathcal{L}_{\bm\phi}^{\mathrm{PVM}}}
(\bm\phi;\bm\theta_W^{*},\bm\theta_U^{*})$ (Eq.\ref{eq:sbo_loss})

\STATE For each $x_j\in\mathcal{D}_{\mathrm{train}}$, predict $\hat{w}_j\gets s_{\bm{\phi}^{*}}(x_j)$

\RETURN Predicted data values
$\hat{\mathbf{w}}_{\mathrm{pool}}
=
\{\hat{w}_j\}_{x_j\in\mathcal{D}_{\mathrm{pool}}}$

\end{algorithmic}
\end{algorithm}

\textbf{Incorporating a Selection Network.}
We extend ScaleBiO by replacing the independent data weights $w_i$ with the data selection meta-network $s_{\bm\phi}(x_i)$. This yields the ScaleBiO meta-network objective (SBO):
\begin{equation}
    \smash{\mathcal L_{\bm\phi}^{\mathrm{SBO}}
    (\bm\phi;\bm\theta_W,\bm\theta_U)}
    =
    \smash{\left(\sum\nolimits_{i=1}^{N_{\mathrm{train}}}
    s_{\bm\phi}(x_i)\right)^{-1}}
    \smash{\sum\nolimits_{i=1}^{N_{\mathrm{train}}}
    s_{\bm\phi}(x_i)}
    \big[
        \ell_i(\bm\theta_W)-\ell_i(\bm\theta_U)
    \big].
    \label{eq:sbo_loss}
\end{equation}


However, empirically optimizing the meta-network with $\smash{\mathcal{L}_{\bm\phi}^{\mathrm{SBO}}}$ leads to unstable optimization
(Figure~\ref{fig:weight_dynamics}) and poor generalization
(Tables~\ref{tab:shortcut_generalization} and
\ref{tab:ablation_selector_loss}). We find this can be simply addressed by Pointwise Value Matching, presented below. 


\textbf{Pointwise Value Matching.} The key idea is to use the loss difference between the two LLMs as a proxy for sample
utility. Recall that $\bm\theta_U$ is trained only on
$\mathcal D_{\mathrm{train}}$, whereas $\bm\theta_W$ is additionally guided by
$\mathcal D_{\mathrm{val}}$. For each training example $x_i$, we define
$\Delta_i(\bm\theta_U,\bm\theta_W):=\ell_i(\bm\theta_U)-\ell_i(\bm\theta_W)$.
A larger $\Delta_i$ indicates a greater loss reduction under the
validation-guided model relative to the training-only model, suggesting
greater relevance to the target validation objective. We use this difference
as a pseudo-label and train the meta-network to match it:
\begin{equation}
\smash{\mathcal{L}_{\bm\phi}^{\mathrm{PVM}}(\bm\phi;\bm\theta_W,\bm\theta_U)
=
\smash{{N_{\mathrm{train}}}^{-1}}
\sum\nolimits_{i=1}^{N_{\mathrm{train}}}
\left(
s_{\bm\phi}(x_i)
-
\Delta_i(\bm\theta_U,\bm\theta_W)
\right)^2}.\end{equation}

\textbf{Sequential Training of \shortname.}
Unlike ScaleBiO, \shortname separates the training of the three models. Specifically, using uniform sample weights, we obtain $\bm\theta_U^*$ by minimizing
$\mathcal L(\bm\theta_U;\mathcal D_{\mathrm{train}})$ and $\bm\theta_W^*$ by
minimizing
$\mathcal L(\bm\theta_W;\mathcal D_{\mathrm{val}})
+\alpha\mathcal L(\bm\theta_W;\mathcal D_{\mathrm{train}})$.
We then freeze both models, compute
$\Delta_i(\bm\theta_U^*,\bm\theta_W^*)$ for each training example, and train the
meta-network by minimizing
$\smash{\mathcal L_{\bm\phi}^{\mathrm{PVM}}
(\bm\phi;\bm\theta_W^*,\bm\theta_U^*)}$. This sequential design has two advantages over ScaleBiO. First, it avoids simultaneously optimizing two LLMs and the
selection network, reducing peak memory consumption. Second, it allows the
pretrained LLMs and the randomly initialized meta-network to use independent
optimization schedules. The trained meta-network
directly scores unseen examples in $\mathcal D_{\mathrm{pool}}$ without
further optimization. The complete procedure of \shortname is
summarized in Algorithm~\ref{alg:ours}.

\vspace{-0.5em}
\section{Why the ScaleBiO Loss Performs Poorly}
\label{sec:selector_objective}
\vspace{-0.5em}
We next examine why optimizing with the ScaleBiO objective can fail to produce a transferable meta-network. We identify two
properties of the ScaleBiO objective: competitive weight suppression,
which weakens sample discrimination, and a persistent preference for
easy-to-learn features, which can encourage shortcut learning.

\subsection{Data Weight Suppression and Loss of Discrimination}
\label{sec:weight_concentration}

Our first observation is that the ScaleBiO objective drives most predicted
weights toward zero, alongside losing the power of identifying good data (Figure~\ref{fig:weight_dynamics}). This occurs under both joint training in ScaleBiO and sequential training in \shortname when training with $\smash{\mathcal L_{\bm\phi}^{\mathrm{SBO}}}$. Our analysis below strictly applies to sequential training in \shortname, where the
two LLMs are fixed. Empirically, we observe qualitatively similar dynamics
under the joint training in ScaleBiO.

We first treat the weights of training examples as independent
hyperparameters. For the selection network, $w_i=s_{\bm\phi}(x_i)$ is updated through shared parameters $\bm\phi$. However, the update directions on $w_i$ with and without $\bm\phi$ are similar. With $\smash{\Delta_i^*:=\Delta_i(\bm\theta_U^*,\bm\theta_W^*)}$ and $N=N_{\mathrm{train}}$,
the ScaleBiO objective and its gradient with respect to each data weight are
\begin{equation}
    \mathcal L_{\bm\phi}^{\mathrm{SBO}}
    =-\sum\nolimits_{i=1}^{N}p_i\Delta_i^*,
    \qquad
    \smash{\frac{\partial\mathcal L_{\bm\phi}^{\mathrm{SBO}}}{\partial w_i}}
    =-\frac{\Delta_i^*-\bar\Delta^*}{S},
    \label{eq:sb_weight_gradient}
\end{equation}
where $\smash{S=\sum_{j=1}^{N}w_j}$ is the total weight,
$p_i=w_i/S$ the normalized weight, and
$\smash{\bar\Delta^*=\sum_i p_i\Delta_i^*}$ the weighted-average pseudo-label.
The derivation is provided in Appendix~\ref{app:sbo_gradient}.

The gradient reveals a competitive mechanism:
examples are up-weighted when $\Delta_i^*>\bar\Delta^*(t)$ and down-weighted
otherwise. As examples with larger pseudo-labels gain relative weight, the
weighted mean rises, making it harder for other examples to receive positive
updates. This can produce a \emph{winner-takes-most} allocation, in which a single weight dominates the total mass. Given a lower bound $\delta_0\geq0$ of data weight, the following theorem is given by \citet{du2026difficulty} (proof reproduced in Appendix \ref{app:weight_concentration_proof}).



\begin{theorem}[Concentration of normalized data weights $p$]
\label{thm:weight_concentration}
Assume there is a unique index $K$
such that $\Delta_K^*>\Delta_i^*$ for all $i\neq K$, then $\lim_{t\to\infty}p_K(t)=1$ and $\lim_{t\to\infty}p_i(t)=0$. For $\delta_0=0$, this limit is reached in finite time.
For $\delta_0>0$, $1-p_K(t)=O(t^{-1/3})$.
\end{theorem}

\begin{wrapfigure}{R}{0.4\linewidth}
    \centering
    \includegraphics[width=\linewidth]{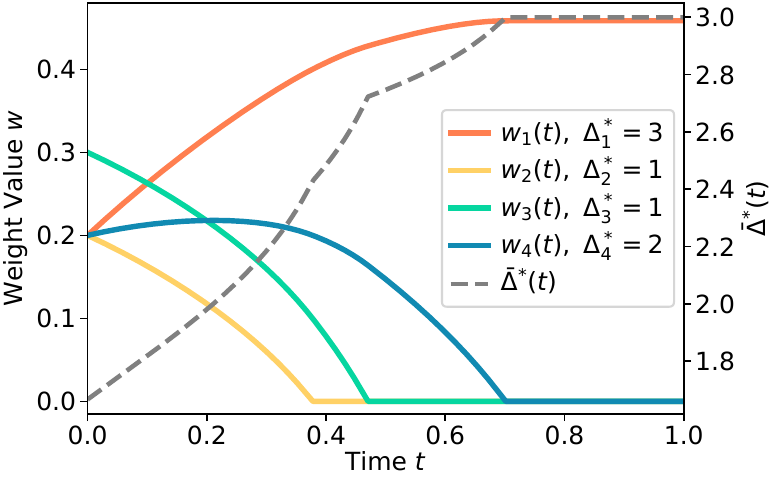}
    \vspace{-20pt}
    \caption{
    Simulation of the weight dynamics in Theorem \ref{thm:weight_concentration}.
    Curves show the weights $w_i(t)$, with fixed pseudo-labels $\Delta_i^*$ indicated in the legend. Initial weights are $(0.2,0.2,0.3,0.2)$,
    and $\delta_0$ is $10^{-6}$.
    }
    \vspace{-10pt}
    \label{fig:weight_simulation}
\end{wrapfigure}

Different from \citet{du2026difficulty}, who use weight concentration to analyze gradient
signal-to-noise ratios, we examine its implications for sample discrimination.
The ScaleBiO objective eventually allocates nearly all weight to one training data point. At the beginning, the meta-network can differentiate high-quality data points from low-quality data points by assigning them different weights. As training progresses, however, its outputs become less discriminative. This creates a strange phenomenon that more training leads to worsened data selection even on the training set itself. 


\textbf{Empirical verification.} Figure~\ref{fig:weight_simulation} illustrates this behavior through a numerical simulation.
The weight $w_1$, associated with the largest pseudo-label, becomes dominant,
while the remaining weights reach the lower bound.
Notably, $w_4$ initially increases but later decreases. Between $t=0.4$ and $t=0.6$, the weights can separate good data with the highest pseudo-labels, but this period is transient. Some data selection can be achieved by carefully tuning early stopping (consistent with Figure 9 of \cite{calian2026datarater}) but it requires significant trial and error. 

To verify these dynamics, we construct a controlled dataset by mixing GSM8K and unsafe examples at a $2{:}1$ ratio, together with a disjoint unseen set of the same composition. This way, we know exactly what data points are unsafe and should be discarded. 
We rank examples by predicted weights and evaluate how accurately the bottom
one-third identifies unsafe examples on both seen and unseen data.
Experimental details are provided in Appendix~\ref{app:toy_exp_details}. Figure~\ref{fig:weight_dynamics} compares ScaleBiO with
$\smash{\mathcal{L}_{\bm\phi}^{\mathrm{SBO}}}$, \shortname with
$\smash{\mathcal{L}_{\bm\phi}^{\mathrm{SBO}}}$, and \shortname with
$\smash{\mathcal{L}_{\bm\phi}^{\mathrm{PVM}}}$.

\begin{figure*}[t]
    \centering
    \includegraphics[width=\textwidth]{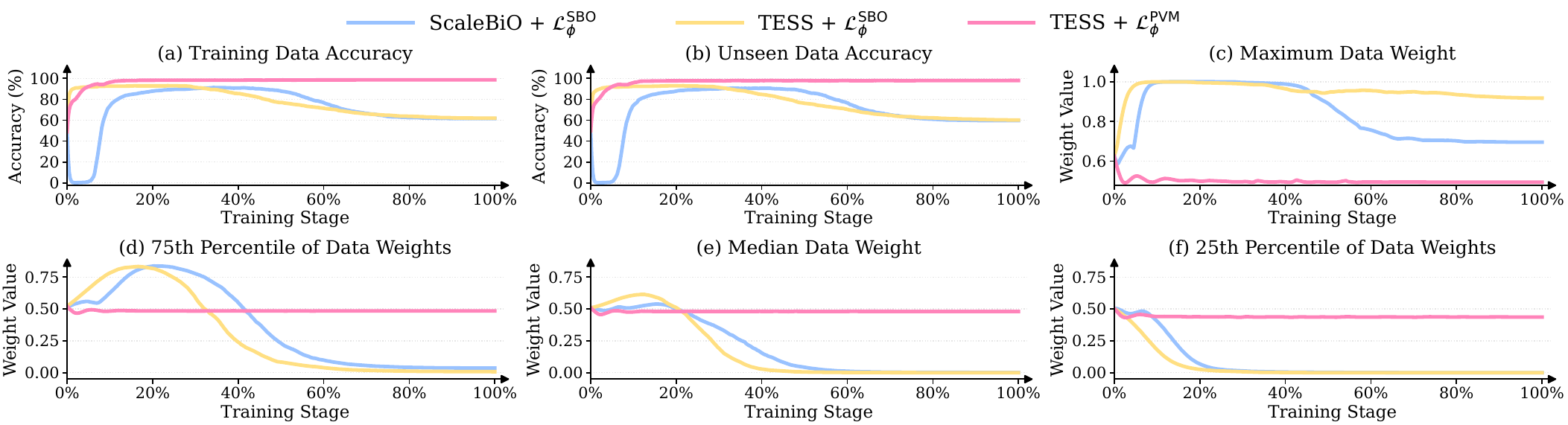}
    \vspace{-16pt}
    \caption{Training dynamics of ScaleBiO and \shortname under different meta-network
    objectives. With SBO loss, most predicted
    weights shrink toward zero as accuracy declines.
    In contrast, PVM loss maintains
    nonvanishing weight quantiles and high accuracy.
    }
    \label{fig:weight_dynamics}
\end{figure*}

Under $\smash{\mathcal{L}_{\bm\phi}^{\mathrm{SBO}}}$, the 25th percentile,
median, and 75th percentile of the predicted weights progressively approach
zero, consistent with the competitive suppression analyzed above. The maximum weight initially increases but later decreases. This is expected
because a shared selection network typically produces smoother outputs than a
one-hot vector~\citep{ma2021linear}, causing even dominant weights to decrease
as most outputs are suppressed.

Weight suppression accompanies declining discrimination.
For ScaleBiO, seen and unseen accuracy first improve but eventually fall to
around $60\%$. \shortname with the same
$\smash{\mathcal{L}_{\bm\phi}^{\mathrm{SBO}}}$ exhibits a similar long-term decline,
although it improves immediately at the beginning of training. We attribute
this early difference to the quality of the supervision signal: ScaleBiO
updates the meta-network while the two LLMs are still being optimized, whereas
\shortname computes loss gaps from well-trained, fixed LLMs. In contrast, \shortname with $\smash{\mathcal{L}_{\bm\phi}^{\mathrm{PVM}}}$ maintains
stable, nonvanishing weight quantiles and reaches $98.58\%$ seen accuracy and $97.96\%$ unseen accuracy.

\subsection{Easy-Feature Dominance and Shortcut Learning}
\label{sec:shortcut_learning}

The previous subsection links weight suppression to reduced sample
discrimination. We now examine a second limitation: ScaleBiO can continue
favoring easy-to-learn features after they have been substantially fitted,
whereas PVM can shift learning toward harder features that explain the
remaining prediction error. This distinction matters for generalization
when an easy feature is predictive on the training data but
unreliable on unseen data, thereby acting as a shortcut
\citep{hermann2024foundations,geirhos2020shortcut}.

Denote the empirical covariance between random variables $u$ and $v$ as 
$\widehat{\operatorname{Cov}}_N(u,v)
:=N^{-1}\sum_{i=1}^{N}(u_i-\bar u)(v_i-\bar v).$
Using fixed pseudo-labels $\Delta_i=\Delta_i(\bm\theta_U^*,\bm\theta_W^*)$, for two data features $e$ and $h$, we denote their covariance with pseudo-labels as
\begin{equation}
c_e=\widehat{\operatorname{Cov}}_N(\Delta,e),
\qquad
c_h=\widehat{\operatorname{Cov}}_N(\Delta,h).
\end{equation}
We call $e$ the \emph{easy feature} and $h$ the \emph{hard feature} if $|c_e|>|c_h|$. For each training example $x_i$, let $e_i$ and $h_i$ be its easy and hard features, both standardized to zero mean and unit standard deviation. For simplicity, we analyze a one-layer meta-network with parameters $\bm\phi=(\beta_0,\beta_e,\beta_h)$ to predict weights for $x_i$. It predicts $w_i=s_{\bm\phi}(x_i)=a_i$, where $a_i=\beta_0+\beta_e e_i+\beta_h h_i$ is the output logit.

\begin{table}[t]
\centering
\scriptsize
\setlength{\tabcolsep}{4.5pt}
\caption{
Safe--unsafe data classification accuracy under different shortcut strengths. $\mathrm{Cov}$ and $\rho$ denote the empirical covariance and Pearson correlation coefficient between the shortcut feature and the label, respectively. Early, middle, and late stages correspond to $20\%$, $50\%$, and
$80\%$ of meta-network training, respectively.
$\smash{\mathrm{Acc}_{\mathrm{S}}}$ and
$\smash{\mathrm{Acc}_{\mathrm{U}}}$ denote seen and unseen accuracy, and
$\smash{\mathrm{Gap}=\mathrm{Acc}_{\mathrm{S}}-\mathrm{Acc}_{\mathrm{U}}}$
denotes the generalization gap.
}
\label{tab:shortcut_generalization}

\begin{tabular}{c|c|ccc|ccc|ccc|c}
\toprule
\multirow{2}{*}[-0.9ex]{\shortstack{Shortcut\\($\mathrm{Cov}$ / $\rho$})}
& \multirow{2}{*}{Method}
& \multicolumn{3}{c|}{Early Stage}
& \multicolumn{3}{c|}{Middle Stage}
& \multicolumn{3}{c|}{Late Stage}
& \multirow{2}{*}[-0.5ex]{\shortstack{Avg.\\Gap}} \\
\cmidrule(lr){3-5}
\cmidrule(lr){6-8}
\cmidrule(lr){9-11}
&
& $\mathrm{Acc}_{\mathrm{S}}$
& $\mathrm{Acc}_{\mathrm{U}}$
& $\mathrm{Gap}$
& $\mathrm{Acc}_{\mathrm{S}}$
& $\mathrm{Acc}_{\mathrm{U}}$
& $\mathrm{Gap}$
& $\mathrm{Acc}_{\mathrm{S}}$
& $\mathrm{Acc}_{\mathrm{U}}$
& $\mathrm{Gap}$
& \\
\midrule

\multirow{3}{*}{\shortstack{Strong\\0.17 / 0.57}}
& TESS + $\mathcal{L}_{\bm\phi}^{\mathrm{PVM}}$
& \textbf{98.84}
& \textbf{79.38}
& \underline{19.46}
& \textbf{99.56}
& \textbf{84.62}
& \textbf{14.94}
& \textbf{99.56}
& \textbf{84.97}
& \textbf{14.59}
& \textbf{16.33} \\

& TESS + $\mathcal{L}_{\bm\phi}^{\mathrm{SBO}}$
& \underline{97.96}
& 72.44
& 25.52
& 90.22
& 52.26
& 37.96
& \underline{84.71}
& \underline{46.31}
& \underline{38.40}
& 33.96 \\

& ScaleBiO + $\mathcal{L}_{\bm\phi}^{\mathrm{SBO}}$
& 93.42
& \underline{78.84}
& \textbf{14.58}
& \underline{90.76}
& \underline{57.73}
& \underline{33.03}
& 76.00
& 32.44
& 43.56
& \underline{30.39} \\

\midrule

\multirow{3}{*}{\shortstack{Weak\\0.06 / 0.41}}
& TESS + $\mathcal{L}_{\bm\phi}^{\mathrm{PVM}}$
& \textbf{98.04}
& \textbf{91.47}
& \textbf{6.57}
& \textbf{98.67}
& \textbf{93.51}
& \textbf{5.16}
& \textbf{98.76}
& \textbf{93.60}
& \textbf{5.16}
& \textbf{5.63} \\

& TESS + $\mathcal{L}_{\bm\phi}^{\mathrm{SBO}}$
& \underline{96.53}
& 80.71
& 15.82
& 87.82
& 65.78
& 22.04
& \underline{84.00}
& \underline{60.53}
& 23.47
& 20.44 \\

& ScaleBiO + $\mathcal{L}_{\bm\phi}^{\mathrm{SBO}}$
& 91.64
& \underline{82.40}
& \underline{9.24}
& \underline{90.13}
& \underline{71.73}
& \underline{18.40}
& 75.56
& 55.56
& \underline{20.00}
& \underline{15.88} \\

\bottomrule
\end{tabular}
\end{table}

To examine the learning preference for features, we compare the gap between gradients of easy and hard features in PVM and ScaleBiO training.




\begin{theorem}[Contrasting gradient-gap dynamics]
\label{thm:feature_learning_comparison}
At training step $t$, denote the feature gradients by
$\smash{G_k^q(\bm\phi_t):=\partial\mathcal L_{\bm\phi_t}^q/\partial\beta_{k,t}}$,
where $k\in\{e,h\}$ and $q\in\{\mathrm{SBO},\mathrm{PVM}\}$. Define gradient gap $ D_t^q=\left(G_e^q(\bm{\phi}_t)\right)^2-\left(G_h^q(\bm{\phi}_t)\right)^2$, then $D_t^{\mathrm{PVM}}$ decreases exponentially with $t$ while $D_t^{\mathrm{SBO}}$ remains non-decreasing. Specifically, when initializing $\beta_e=\beta_h=0$ and $\beta_0>0$, there exists a training step $T>0$ and $\lambda\in(0,1)$, such that $\forall t<T$, 
\begin{equation}
    D_t^{\mathrm{PVM}}=(1-\lambda)^t D_0^{\mathrm{PVM}},
    \qquad
    D_t^{\mathrm{SBO}}\ge D_{t-1}^{\mathrm{SBO}}.
    \label{eq:gradient-gap difference}
\end{equation}

\end{theorem}

The proof is provided in Appendix~\ref{app:feature_learning_proof}. This comparison explains
how PVM can gradually balance the information learned from easy and hard features, while ScaleBiO continues to prefer easy features. 
This links the training objective to shortcut learning. When the easy
feature is predictive on the training data but unreliable on unseen data,
ScaleBiO's persistent preference can reinforce reliance on this shortcut.
In contrast, PVM mitigates the gradient gap between the easy feature and the hard feature. This provides a
mechanism for \shortname to move beyond training-specific shortcuts and
improve data selection on unseen examples.

\textbf{Empirical verification.}
To test shortcut reliance, we add shortcut features to the seen training
data and flip them on unseen data (Appendix~\ref{app:toy_exp_details}). We use the gap between seen and unseen dataset accuracy as a proxy for generalization, with smaller gaps
indicating better generalization on unseen data. Table~\ref{tab:shortcut_generalization} shows clear
differences in training dynamics. Under both shortcut strengths, \shortname
with $\smash{\mathcal L_{\bm\phi}^{\mathrm{PVM}}}$ improves unseen accuracy and
reduces the gap from early to late training. In contrast,
$\smash{\mathcal L_{\bm\phi}^{\mathrm{SBO}}}$ leads to lower unseen accuracy and
larger gaps over training in both frameworks. These results are consistent
with our analysis that SBO persistently learning more shortcut features, and
support PVM as a more effective objective for learning transferable
meta-networks. Moreover, all variants exhibit larger average gaps under strong
shortcuts, indicating that stronger spurious correlations make
generalization more difficult.
\vspace{-1em}
\section{Experiments}
\vspace{-0.5em}
\subsection{LLM Safety}
\label{sec:safety_exp}

We evaluate whether a meta-network can identify training examples that increase harmful behavior finetuning a target LLM. 
We use Alpaca~\citep{alpaca,peng2023instruction} and
Dolly~\citep{DatabricksBlog2023DollyV2} as candidate pools
$\mathcal D_{\mathrm{pool}}$.
The validation set from \citet{he2024your} contains 100 harmful
instructions paired with safe refusal responses.
We finetune the target LLM on 1,000 lowest-scoring examples and evaluate
it on DirectHarm4 (DH4)~\citep{lyu2024keeping},
HarmBench (HB)~\citep{mazeika2024harmbench}, and
HEx-PHI (HEx)~\citep{qi2024fine}.
We use LlamaGuard 3~\citep{dubey2024llama3herdmodels} as the classifier to compute attack success rate (ASR) from generated responses; higher ASR indicates more effective harmful-data selection.
The target models are Llama-3-8B-Instruct,
Llama-3.1-8B-Instruct~\citep{dubey2024llama3herdmodels}, and
Qwen2.5-7B-Instruct~\citep{qwen_qwen25_2025}. Implementation details are provided in Appendix~\ref{app:experiment_details}.

\textbf{Settings and Baselines.} We consider two settings.
In \textit{Training}, the meta-network is trained and applied on the same
candidate dataset, with
$\mathcal D_{\mathrm{train}}=\mathcal D_{\mathrm{pool}}$.
In \textit{Generalization}, it is trained on Alpaca and directly applied to Dolly,
or vice versa, with
$\mathcal D_{\mathrm{train}}\cap\mathcal D_{\mathrm{pool}}=\varnothing$.
We compare against the random selection baseline and two groups of task-agnostic methods. GradSafe~\citep{xie2024gradsafe} and
Bi-Anchoring (Bi-Anchor)~\citep{he2024your} are heuristic methods that rank examples using distances relative to the validation set. SEAL~\citep{shen2025seal} and
ScaleBiO~\citep{pan2025scalebio} are meta-learning methods that learn data weights. We create two variants for each. SEAL$_w$ and SBO$_w$ optimize one weight per
example, whereas SEAL$_\phi$ and SBO$_\phi$ train a
meta-network $s_{\bm\phi}$ to predict weights. 

\begin{table*}[t]
\centering
\scriptsize
\setlength{\tabcolsep}{3.5pt}
\caption{
Attack success rate (ASR, \%; higher is better) for harmful-data selection
on Alpaca and Dolly. Best and second-best results are \textbf{bold} and
\underline{underlined}, ranked separately within two settings. 
}
\label{tab:safety_main}

\resizebox{\textwidth}{!}{%
\begin{tabular}{ll|c|cc|ccccc|ccc}
\toprule
\multicolumn{2}{c|}{Setting}
& \multicolumn{8}{c|}{Training}
& \multicolumn{3}{c}{Generalization} \\
\cmidrule(lr){1-2}
\cmidrule(lr){3-10}
\cmidrule(lr){11-13}
\multirow{2}{*}{Dataset}
& \multirow{2}{*}{Bench}
& \multirow{2}{*}{Random}
& \multicolumn{2}{c|}{Task-agnostic Heuristics}
& \multicolumn{8}{c}{Learnable Weighting} \\
\cmidrule(lr){4-5}
\cmidrule(lr){6-13}
& &
& GradSafe
& Bi-Anchor
& SEAL$_w$
& SBO$_w$
& SEAL$_\phi$
& SBO$_\phi$
& \shortname
& SEAL$_\phi$
& SBO$_\phi$
& \shortname \\
\midrule

\rowcolor{gray!10}
\multicolumn{13}{c}{
    \rule[-0.9ex]{0pt}{3.0ex}
    \textit{Target LLM: Llama3-8B-Instruct}
} \\
\multirow{3}{*}{Alpaca}
& DH4
& 25.00 & 28.00 & \textbf{49.00}
& 26.75 & \underline{38.25} & 12.50 & 12.50 & 36.75
& \underline{8.50} & 6.50 & \textbf{35.50} \\
& HB
& 15.00 & 16.00 & \textbf{35.00}
& 13.50 & \underline{21.00} & 9.00 & 7.00 & 20.50
& \underline{6.00} & 4.00 & \textbf{25.00} \\
& HEx
& 6.55 & 8.97 & \textbf{24.58}
& 6.90 & 10.69 & 5.86 & 3.44 & \underline{11.38}
& 3.10 & \underline{3.44} & \textbf{15.52} \\
\cmidrule(lr){1-13}
\multirow{3}{*}{Dolly}
& DH4
& 55.25 & \underline{75.00} & 74.50
& 60.75 & 69.25 & 73.00 & 74.50 & \textbf{82.25}
& \underline{56.25} & 43.00 & \textbf{71.25} \\
& HB
& 39.25 & \underline{82.00} & 75.00
& 59.50 & 73.00 & 71.25 & 69.50 & \textbf{85.50}
& \underline{43.50} & 27.50 & \textbf{59.50} \\
& HEx
& 31.38 & \underline{74.14} & 67.59
& 42.76 & 53.44 & 61.03 & 63.44 & \textbf{81.72}
& \underline{30.69} & 22.07 & \textbf{52.41} \\
\cmidrule(lr){1-13}
\multicolumn{2}{c|}{Average}
& 28.74 & 47.35 & \textbf{54.28}
& 35.03 & 44.27 & 38.77 & 38.40 & \underline{53.02}
& \underline{24.67} & 17.75 & \textbf{43.20} \\
\midrule

\rowcolor{gray!10}
\multicolumn{13}{c}{
    \rule[-0.9ex]{0pt}{3.0ex}
    \textit{Target LLM: Llama3.1-8B-Instruct}
} \\
\multirow{3}{*}{Alpaca}
& DH4
& 22.50 & 7.50 & 11.00
& \underline{27.75} & 20.50 & 8.50 & 5.50 & \textbf{40.25}
& 2.25 & \underline{11.00} & \textbf{34.25} \\
& HB
& \textbf{18.50} & 5.00 & 12.50
& \underline{13.00} & 8.00 & 7.50 & 3.50 & \textbf{18.50}
& 2.50 & \underline{7.00} & \textbf{16.50} \\
& HEx
& \underline{8.97} & 3.45 & 3.10
& 6.90 & 2.41 & 2.76 & 1.03 & \textbf{12.41}
& 1.03 & \underline{4.83} & \textbf{9.66} \\
\cmidrule(lr){1-13}
\multirow{3}{*}{Dolly}
& DH4
& 54.00 & 59.50 & 67.25
& 71.75 & 63.75 & 64.00 & \underline{74.50} & \textbf{81.00}
& \underline{71.50} & 60.25 & \textbf{80.25} \\
& HB
& 51.00 & 60.50 & 50.50
& \underline{65.00} & 58.50 & 55.00 & 60.00 & \textbf{80.00}
& \underline{69.00} & 50.50 & \textbf{81.00} \\
& HEx
& 29.30 & 33.79 & 40.00
& 38.62 & 33.10 & 33.45 & \underline{42.07} & \textbf{60.34}
& \underline{40.69} & 28.28 & \textbf{71.38} \\
\cmidrule(lr){1-13}
\multicolumn{2}{c|}{Average}
& 30.71 & 28.29 & 30.73
& \underline{37.17} & 31.04 & 28.54 & 31.10 & \textbf{48.75}
& \underline{31.16} & 26.98 & \textbf{48.84} \\
\midrule

\rowcolor{gray!10}
\multicolumn{13}{c}{
    \rule[-0.9ex]{0pt}{3.0ex}
    \textit{Target LLM: Qwen2.5-7B-Instruct}
} \\
\multirow{3}{*}{Alpaca}
& DH4
& 27.50 & 26.00 & \textbf{44.50}
& 20.00 & 27.00 & \underline{28.25} & 11.75 & \textbf{44.50}
& 8.75 & \underline{11.00} & \textbf{38.75} \\
& HB
& 11.00 & 10.00 & \textbf{24.50}
& 9.00 & 8.00 & 17.50 & 3.50 & \underline{23.50}
& \underline{3.00} & 2.50 & \textbf{16.00} \\
& HEx
& 13.10 & 12.07 & \underline{24.80}
& 6.55 & 8.28 & 9.31 & 6.56 & \textbf{24.83}
& \underline{7.24} & 6.56 & \textbf{18.28} \\
\cmidrule(lr){1-13}
\multirow{3}{*}{Dolly}
& DH4
& 50.50 & \underline{66.50} & 60.50
& 49.75 & 65.75 & 60.25 & 59.75 & \textbf{86.50}
& 48.00 & \underline{62.25} & \textbf{84.50} \\
& HB
& 36.00 & 60.00 & 60.50
& \underline{65.50} & 61.50 & 43.00 & 45.50 & \textbf{87.00}
& 34.00 & \underline{50.50} & \textbf{83.50} \\
& HEx
& 32.41 & \underline{51.03} & 42.07
& \underline{51.03} & 50.34 & 45.86 & 44.48 & \textbf{88.62}
& 27.59 & \underline{29.66} & \textbf{80.00} \\
\cmidrule(lr){1-13}
\multicolumn{2}{c|}{Average}
& 28.42 & 37.60 & \underline{42.81}
& 33.64 & 36.81 & 34.03 & 28.59 & \textbf{59.16}
& 21.43 & \underline{27.08} & \textbf{53.51} \\
\midrule

\multicolumn{2}{c|}{Overall Average}
& 29.29 & 37.75 & \underline{42.61}
& 35.28 & 37.38 & 33.78 & 32.70 & \textbf{53.64}
& \underline{25.76} & 23.94 & \textbf{48.51} \\
\bottomrule
\end{tabular}%
}
\end{table*}

\textbf{Main Results.}
\label{sec:safety_results}
Table~\ref{tab:safety_main} shows the results. Under \textit{Training} settings, \shortname achieves the strongest performance, exceeding the strongest learning  baseline
SBO$_w$ by $16.27\%$ and the strongest heuristic
Bi-Anchor by $11.04\%$. Although Bi-Anchor remains stronger on
Llama3-Alpaca, \shortname achieves the best in
14 of the 18 comparisons. These gains stem from the optimization stability of \shortname. In contrast, simply replacing individual weights with a shared meta-network does not consistently help: SEAL$\bm\phi$ and SBO$_\phi$ underperform their weight-based counterparts by $1.50$ and $4.68$ points on average.
This suggests that the training objective is critical for effective selection.

Under the \textit{Generalization} setting, \shortname outperforms SEAL$_\phi$ and SBO$_\phi$ in all 18 evaluations, surpassing SEAL$_\phi$ and SBO$_\phi$ by $22.76$ and $24.58$ points in average ASR, respectively, while dropping only $5.13\%$ from its training
performance. These results support our analysis that PVM learns more transferable selection signals rather than relying on training-specific shortcuts.

\textbf{Small-to-Large Model Transfer.}
We further evaluate whether a meta-network trained using a smaller LLM can
select useful data for a larger model. We train the meta-network with Qwen2.5-0.5B-Instruct and use it to select data for fine-tuning Qwen2.5-7B-Instruct.
As shown in Table~\ref{tab:safety_transfer}
(Appendix~\ref{app:small2big}), this reduces training time by $5.65$--$6.60\times$ while retaining $77.53\%$ and $83.67\%$ of the corresponding large-model performance under the \textit{Training} and \textit{Generalization} settings. This demonstrates that \shortname can transfer selection signals across model scales with substantially lower training cost. This offers a promising direction for scaling data selection to larger models efficiently.

\textbf{Ablation Study.}
\label{sec:ablation_loss}
We ablate the training losses within the \shortname framework
on Qwen2.5-7B-Instruct, comparing PVM Loss (\textit{defalt}), SBO Loss, KL Loss, together with Pseudo-Label Ranking, which directly ranks examples by $\Delta^*$ without training a meta-network. KL Loss (Appendix~\ref{app:KL_loss}) matches normalized meta-network outputs to a pseudo-label-reweighted target
distribution that favors examples with higher estimated utility.

Under the \textit{Training} setting, \shortname with PVM Loss achieves
the highest ASR across all six dataset--benchmark pairs, outperforming SBO Loss and KL Loss by $21.60\%$ and $16.24\%$, respectively. These gains stem from the more stable optimization of PVM.
Notably, PVM Loss also outperforms \emph{Pseudo-Label Ranking} by
$18.15\%$.
One possible explanation is that the meta-network may implicitly denoise pseudo-labels by learning shared
patterns \citep{liu2020early,li2020gradient}. Under the \textit{Generalization} setting, PVM Loss exceeds SBO Loss and KL Loss by $24.61\%$ and $25.07\%$ on average, respectively.
Moreover, its average ASR decreases by only $5.65$ points
relative to the \textit{Training} setting, compared with drops of
$8.67\%$ for SBO Loss and $14.49\%$ for KL Loss.
The stronger unseen-data performance and smaller degradation indicate that PVM learns a more transferable meta-network.

\begin{table*}[t]
\centering
\scriptsize
\caption{
Ablation of Different Training Losses for the Meta-network on Qwen2.5-7B-Instruct.}
\label{tab:ablation_selector_loss}

\begin{tabular}{@{}llccccccc@{}}
\toprule
\multirow{2}{*}{Setting}
& \multirow{2}{*}{Method}
& \multicolumn{3}{c}{Alpaca}
& \multicolumn{3}{c}{Dolly}
& \multirow{2}{*}{Avg.} \\
\cmidrule(lr){3-5}
\cmidrule(lr){6-8}
&
& DH4 & HB & HEx
& DH4 & HB & HEx
& \\
\midrule

\multirow{4}{*}{Training}
& TESS w/ PVM Loss
& \textbf{44.50}
& \textbf{23.50}
& \textbf{24.83}
& \textbf{86.50}
& \textbf{87.00}
& \textbf{88.62}
& \textbf{59.16} \\

& TESS w/ SBO Loss
& 39.50
& \underline{19.50}
& 22.06
& 58.25
& 40.50
& 45.52
& $37.56_{\text{-21.60}}$ \\

& TESS w/ KL Loss
& \underline{40.75}
& 19.00
& \underline{23.86}
& 60.75
& 53.50
& 59.66
& $\underline{42.92}_{\text{-16.24}}$ \\

& Pseudo-Label Ranking
& 11.25
& 7.50
& 5.86
& \underline{75.75}
& \underline{75.00}
& \underline{70.69}
& $41.01_{\text{-18.15}}$ \\

\midrule

\multirow{3}{*}{Generalization}
& TESS w/ PVM Loss
& \textbf{38.75}
& \textbf{16.00}
& \textbf{18.28}
& \textbf{84.50}
& \textbf{83.50}
& \textbf{80.00}
& \textbf{53.51} \\

& TESS w/ SBO Loss
& \underline{12.00}
& \underline{9.50}
& \underline{7.24}
& 59.25
& \underline{44.00}
& 41.37
& $\underline{28.89}_{\text{-24.61}}$ \\

& TESS w/ KL Loss
& 11.75
& 4.00
& 6.55
& \underline{61.00}
& 43.50
& \underline{43.79}
& $28.43_{\text{-25.07}}$ \\

\bottomrule
\end{tabular}
\end{table*}

\vspace{-0.5em}
\subsection{Targeted Instruction Tuning}
\label{sec:targeted_tuning}
\vspace{-0.5em}
We evaluate if \shortname trained on a small subset can generalize to the full instruction tuning
corpus. 
GSM8K~\citep{cobbe2021training} and
Codex~\citep{chen2021evaluating} are the target tasks. We train the meta-network using
a random 50K-example subset of Tulu V2~\citep{ivison2025large} and a task-specific validation set, then apply it to select training data from the 197K full Tulu V2.
Following \citet{nayak2026critical}, we finetune Llama-2-7B on the
selected examples and report GSM8K accuracy and Codex pass@10.
Implementation details are provided in
Appendix~\ref{app:experiment_details}.
We compare against random selection, heuristic-based methods, and
learnable-weighting baselines.
The heuristic methods, LESS~\citep{xia2024less} and
RDS+~\citep{ivison2025large}, rank candidate examples using predefined
measures of relevance to the target validation set.
All four learnable baselines are trained on the same 50K subset. SEAL$_w$ and SBO$_w$ are restricted to select from their training set.

\begin{table}[t]
\vspace{-1em}
\centering
\scriptsize
\caption{
Llama-2-7B performance after fine-tuning on $N$ selected Tulu V2 examples.
}
\label{tab:targeted_reasoning}

\begin{tabular}{l|cccc|cccc|l}
\toprule
\multirow{2}{*}{Method}
& \multicolumn{4}{c|}{GSM8K}
& \multicolumn{4}{c|}{CodeX}
& \multirow{2}{*}[-0.9ex]{\shortstack{Overall\\Avg.}} \\
\cmidrule(lr){2-5}
\cmidrule(lr){6-9}
& $N$=1K
& $N$=5K
& $N$=10K
& Avg.
& $N$=5K
& $N$=10K
& $N$=20K
& Avg.
& \\
\midrule

Random
& 13.41
& 15.92
& 16.13
& 15.15
& \underline{29.05}
& \underline{27.03}
& 27.70
& 27.93
& 21.54 \\

LESS
& \underline{18.80}
& \underline{22.36}
& 22.59
& \underline{21.25}
& 25.00
& 26.35
& 28.37
& 26.57
& 23.91$_{\text{+2.37}}$ \\

RDS+
& 17.43
& 21.60
& \underline{23.57}
& 20.87
& 28.37
& 26.35
& \textbf{29.73}
& \underline{28.15}
& \underline{24.51}$_{\text{+2.97}}$ \\

SEAL$_w$
& 14.78
& 16.37
& 17.89
& 16.35
& 24.32
& 25.00
& 25.68
& 25.00
& 20.67$_{\text{-0.87}}$ \\

SBO$_w$
& 15.01
& 17.74
& 18.65
& 17.13
& 26.35
& 25.68
& 27.03
& 26.35
& 21.74$_{\text{+0.20}}$ \\

SEAL$_\phi$
& 17.74
& 19.56
& 19.48
& 18.93
& 24.32
& \underline{27.03}
& 26.35
& 25.90
& 22.41$_{\text{+0.87}}$ \\

SBO$_\phi$
& 17.43
& 18.65
& 19.03
& 18.37
& 26.35
& 25.68
& 27.70
& 26.58
& 22.47$_{\text{+0.93}}$ \\

Ours
& \textbf{20.43}
& \textbf{25.47}
& \textbf{25.70}
& \textbf{23.87}
& \textbf{29.73}
& \textbf{28.37}
& \underline{29.05}
& \textbf{29.05}
& \textbf{26.46}$_{\text{+4.92}}$ \\

\bottomrule
\end{tabular}
\vspace{-8pt}
\end{table}

\textbf{Main Results.}
\label{sec:targeted_results}Table~\ref{tab:targeted_reasoning} shows that \shortname achieves the best average performance on both tasks and ranks first in five of six settings, exceeding Random by $4.92$ points overall. \shortname improves over heuristic-based selection,
exceeding LESS and RDS+ in overall average by 2.55\% and
1.95\%, respectively.
The advantage is more pronounced on GSM8K, where \shortname
outperforms the strongest heuristic baseline, LESS, in task-average
performance by 2.62\%. \shortname also consistently outperforms all learnable-weighting baselines. Compared with SEAL$\bm\phi$ and SBO$\bm\phi$, which use the same 50K training subset and full selection pool, \shortname improves the overall average by $4.05$ and $3.99$ points, respectively. This controlled comparison shows that the performance gains come from the new loss rather than merely using a shared meta-network.

Regarding training efficiency, SEAL$_w$ and SBO$_w$ optimize sample-specific weights and thus must be trained on the full 200K pool, whereas \shortname trains on only 50K examples and generalizes to the full pool. As a result, SBO$_w$ and SEAL$_w$ require 59.8 and 39.6 hours, respectively, corresponding to 3.46$\times$ and 1.95$\times$ more training time than the 13.4 hours required by \shortname. SBO$_w$ also incurs higher memory usage due to jointly training multiple models, requiring 188.2 GB compared with 156.8 GB for \shortname, an increase of 31.4,GB (20.03\%).

\vspace{-0.5em}
\section{Conclusion} 
\vspace{-0.5em}
We introduced \shortname, a transferable data-selection framework that learns a selector from sample-level pseudo-labels. We identify weight suppression and easy-feature preference of ScaleBiO objective and address them with PVM. \shortname shows strong transfer across datasets, corpus scales, and model sizes.


\subsection*{AI use statement}

In this work, we used generative AI tools to formulate mathematical claims, assist in the writing of proofs, provide feedback on research methodology or experiments, help develop theoretical models or conceptual frameworks, provide critical ingredients for proving mathematical claims, and assist with translation.
We have not used generative AI tools to generate synthetic data sets, implement methods, support qualitative and thematic data analysis, or interpret results, and propose or refine hypotheses and clean and reformat dataset are not applicable to this work.

Additionally, we used generative AI tools to edit software code, search for information, and edit a research paper to improve readability.
We have not used generative AI tools to create or modify scientific figures or images, suggest experimental parameters, draft parts of a research paper, summarize or analyse existing literature, discover research topics or identify gaps, brainstorm, identify relevant literature, format references, suggest a structure for a research paper, or propose a title or keywords for a research paper.
Formulate questions for surveys or interviews, creation of artifacts, and transcribe recordings of research material are not applicable to this work.

We have reviewed all AI-assisted work. All AI-assisted text, mathematical results, and information-search results were manually verified and revised by the authors as needed.
We take responsibility for the final content of this work, including text, claims or artifacts produced with the aid of generative AI.

\subsection*{Reproducibility statement}

We provide detailed descriptions of the models, datasets, training configurations, and evaluation protocols in Section~\ref{sec:safety_exp} and Appendix~\ref{app:experiment_details}. All datasets and pretrained models used in this work are from prior work and are publicly available. We will also release our code upon acceptance to facilitate reproducibility.

\bibliography{iclr2026_conference}
\bibliographystyle{iclr2026_conference}

\clearpage
\appendix

\startcontents[appendices]
\section*{Appendix Contents}
\printcontents[appendices]{}{1}{%
    \setcounter{tocdepth}{2}%
}
    
\clearpage

\section{Gradient of the ScaleBiO meta-network Objective}
\label{app:sbo_gradient}

\paragraph{Notation.}
Let $\mathcal D_{\mathrm{train}}=\{x_i\}_{i=1}^{N}$, where
$N=N_{\mathrm{train}}$ is the number of training samples.
Let $\ell(\bm\theta;x_i)$ denote the token-averaged cross-entropy loss
of an LLM parameterized by $\bm\theta$ on sample $x_i$.
For the selection network $s_{\bm\phi}$, define
\[
    w_i := s_{\bm\phi}(x_i)>0,
    \qquad
    S := \sum_{j=1}^{N}w_j,
    \qquad
    p_i := \frac{w_i}{S}.
\]
The two LLMs are held fixed at $\bm\theta_U^*$ and $\bm\theta_W^*$.
Define the per-sample pseudo-labels and their weighted average as
\begin{align*}
    \Delta_i^*
    &:=
    \ell(\bm\theta_U^*;x_i)-\ell(\bm\theta_W^*;x_i),
    \\
    \bar{\Delta}^*
    &:=
    \sum_{j=1}^{N}p_j\Delta_j^*
    =
    \frac{\sum_{j=1}^{N}w_j\Delta_j^*}{S}.
\end{align*}
The pseudo-labels $\Delta_i^*$ are constant with respect to $\bm\phi$
and the data weights, whereas $\bar{\Delta}^*$ depends on the weights
through $p_j$.

\paragraph{meta-network objective.}
For the fixed LLMs, the objective in Eq.~\ref{eq:sbo_loss} becomes
\begin{align*}
    \mathcal L_{\bm\phi}^{\mathrm{SBO}}
    &=
    \sum_{j=1}^{N}
    \frac{w_j}{S}
    \big[
        \ell(\bm\theta_W^*;x_j)-\ell(\bm\theta_U^*;x_j)
    \big]
    \\
    &=
    -\sum_{j=1}^{N}p_j\Delta_j^*
    =
    -\frac{\sum_{j=1}^{N}w_j\Delta_j^*}{S}
    =
    -\bar{\Delta}^*.
\end{align*}

\paragraph{Gradient with respect to an individual weight.}
For a fixed index $i$, we differentiate with respect to $w_i$ while
holding all other weights fixed.
Since the pseudo-labels are constant, the derivatives of the numerator
and denominator are
\begin{align*}
    \frac{\partial}{\partial w_i}
    \left(\sum_{j=1}^{N}w_j\Delta_j^*\right)
    &=
    \frac{\partial}{\partial w_i}
    \left(
        w_i\Delta_i^*
        +
        \sum_{j\ne i}w_j\Delta_j^*
    \right)
    =
    \Delta_i^*,
    \\
    \frac{\partial S}{\partial w_i}
    &=
    \frac{\partial}{\partial w_i}
    \left(
        w_i+\sum_{j\ne i}w_j
    \right)
    =
    1.
\end{align*}
Applying the quotient rule and substituting these derivatives gives
\begin{align}
    \frac{\partial\mathcal L_{\bm\phi}^{\mathrm{SBO}}}{\partial w_i}
    &=
    -\frac{
        S\,\frac{\partial}{\partial w_i}
        \left(\sum_{j=1}^{N}w_j\Delta_j^*\right)
        -
        \left(\sum_{j=1}^{N}w_j\Delta_j^*\right)
        \frac{\partial S}{\partial w_i}
    }{S^2}
    \nonumber\\
    &=
    -\frac{
        S\Delta_i^*
        -
        \sum_{j=1}^{N}w_j\Delta_j^*
    }{S^2}
    \nonumber\\
    &=
    -\frac{1}{S}
    \left(
        \Delta_i^*
        -
        \frac{\sum_{j=1}^{N}w_j\Delta_j^*}{S}
    \right)
    \nonumber\\
    &=
    -\frac{\Delta_i^*-\bar{\Delta}^*}{S}.
    \label{eq:app_sbo_weight_gradient}
\end{align}

\section{Proof of Theorem~\ref{thm:weight_concentration}}
\label{app:weight_concentration_proof}

\paragraph{Notation and assumptions.}
Let $N=N_{\mathrm{train}}\geq2$ and fix the pseudo-labels
$\Delta_i^*:=\Delta_i(\bm\theta_U^*,\bm\theta_W^*)$.
We treat the raw weights $\mathbf w=(w_1,\ldots,w_N)$ as
independent variables constrained only by
$w_i\geq\delta_0\geq0$, with $w_i(0)>\delta_0$ for every $i$.
In particular, no upper bound is imposed on the weights.
Assume there is a unique index $K$ such that
$\Delta_K^*>\Delta_i^*$ for all $i\neq K$.

For any feasible $\mathbf w$ with positive total weight, define
\[
    S(\mathbf w)=\sum_{j=1}^{N}w_j,
    \qquad
    p_i(\mathbf w)=\frac{w_i}{S(\mathbf w)},
    \qquad
    \bar\Delta^*(\mathbf w)
    =\sum_{i=1}^{N}p_i(\mathbf w)\Delta_i^*.
\]
Writing the ScaleBiO objective as a function of the independent
weights gives
\[
    \mathcal L^{\mathrm{SBO}}(\mathbf w)
    =-\bar\Delta^*(\mathbf w),
    \qquad
    \frac{\partial\mathcal L^{\mathrm{SBO}}}{\partial w_i}
    =-\frac{\Delta_i^*-\bar\Delta^*(\mathbf w)}
    {S(\mathbf w)}.
\]

\paragraph{Gradient descent and continuous-time dynamics.}
To establish the theorem, we first specify the weight updates
and their continuous-time counterpart.
Let $n$ index discrete iterations and let $\eta_w>0$ be the
learning rate.
Projected gradient descent under the constraint
$w_i\geq\delta_0$ yields
\begin{equation}
    w_i^{(n+1)}
    =
    \max\left\{
        \delta_0,\,
        w_i^{(n)}
        +
        \frac{\eta_w}{S^{(n)}}
        \bigl(\Delta_i^*-\bar\Delta^{*(n)}\bigr)
    \right\},
    \label{eq:sb_weight_update}
\end{equation}
where $S^{(n)}=S(\mathbf w^{(n)})$ and
$\bar\Delta^{*(n)}=\bar\Delta^*(\mathbf w^{(n)})$.
The projection prevents an update from taking a weight below
$\delta_0$.

The corresponding continuous-time dynamics, with the learning
rate absorbed into time through $t=n\eta_w$, are described by
the projected gradient flow
\begin{equation}
    \dot w_i(t)
    =
    \begin{cases}
        \displaystyle
        \frac{\Delta_i^*-\bar\Delta^*(t)}{S(t)},
        & w_i(t)>\delta_0,\\[6pt]
        \displaystyle
        \max\left\{
            0,\,
            \frac{\Delta_i^*-\bar\Delta^*(t)}{S(t)}
        \right\},
        & w_i(t)=\delta_0,
    \end{cases}
    \label{eq:sbo_projected_flow}
\end{equation}
where $S(t)=S(\mathbf w(t))$,
$p_i(t)=p_i(\mathbf w(t))$, and
$\bar\Delta^*(t)=\bar\Delta^*(\mathbf w(t))$.
At the boundary, negative directions are suppressed while
positive directions remain admissible.
The proof below establishes the concentration result for
this continuous-time flow.

\begin{proof}
Consider an absolutely continuous solution of
Eq.~\ref{eq:sbo_projected_flow}.
All differential identities and inequalities below hold
almost everywhere, including across boundary transitions.

Define the pseudo-label range
\[
    R_\Delta
    :=\Delta_K^*-\min_i\Delta_i^*>0
\]
and the set of coordinates following the unprojected direction:
\[
    \mathcal I(t)
    :=
    \left\{
        i:
        w_i(t)>\delta_0
        \ \text{or}\
        \Delta_i^*>\bar\Delta^*(t)
    \right\}.
\]
Then
\[
    \dot w_i(t)
    =
    \begin{cases}
        (\Delta_i^*-\bar\Delta^*(t))/S(t),
        & i\in\mathcal I(t),\\
        0,
        & i\notin\mathcal I(t).
    \end{cases}
\]

\paragraph{1. Positivity and growth of the normalization.}
Since the weights are nonnegative, $\bar\Delta^*(t)$ is a
convex combination of the pseudo-labels.
Thus $\bar\Delta^*(t)\leq\Delta_K^*$ and $\dot w_K(t)\geq0$,
which implies
\begin{equation}
    S(t)\geq w_K(t)\geq w_K(0)>0.
    \label{eq:app_positive_normalization}
\end{equation}
In particular, $w_K(t)>\delta_0$, so
$K\in\mathcal I(t)$ for all $t$.

Moreover,
$|\Delta_i^*-\bar\Delta^*(t)|\leq R_\Delta$.
Hence $|\dot w_i(t)|\leq R_\Delta/w_K(0)$, excluding
finite-time blow-up.
Summing the coordinate derivatives gives
\begin{align}
    \frac{d}{dt}S(t)^2
    &=
    2S(t)\sum_i\dot w_i(t)
    \nonumber\\
    &=
    2\sum_{i\in\mathcal I(t)}
    \bigl(\Delta_i^*-\bar\Delta^*(t)\bigr)
    \leq2NR_\Delta.
\end{align}
Therefore,
\begin{equation}
    S(t)^2\leq S(0)^2+2NR_\Delta t.
    \label{eq:app_normalization_growth}
\end{equation}
By comparison with the reciprocals of this upper bound and
its square root, respectively,
\begin{equation}
    \int_0^\infty\frac{dt}{S(t)^2}=\infty,
    \qquad
    \int_0^\infty\frac{dt}{S(t)}=\infty.
    \label{eq:app_divergent_integrals}
\end{equation}

\paragraph{2. The weighted mean converges to the largest
pseudo-label.}
Differentiating the weighted mean yields
\begin{align}
    \frac{d\bar\Delta^*(t)}{dt}
    &=
    \frac{1}{S(t)}
    \sum_i
    \bigl(\Delta_i^*-\bar\Delta^*(t)\bigr)\dot w_i(t)
    \nonumber\\
    &=
    \frac{1}{S(t)^2}
    \sum_{i\in\mathcal I(t)}
    \bigl(\Delta_i^*-\bar\Delta^*(t)\bigr)^2
    \geq0.
    \label{eq:app_weighted_mean_monotone}
\end{align}
Thus $\bar\Delta^*(t)$ is nondecreasing and bounded above
by $\Delta_K^*$.

Let $V(t)=\Delta_K^*-\bar\Delta^*(t)\geq0$.
Since $K\in\mathcal I(t)$,
\begin{equation}
    \dot V(t)\leq-\frac{V(t)^2}{S(t)^2}.
    \label{eq:app_gap_differential_inequality}
\end{equation}
If $V$ reaches zero, it remains zero by the monotonicity of
$\bar\Delta^*$.
Otherwise, integrating
$\frac{d}{dt}(1/V(t))\geq1/S(t)^2$ gives
\[
    \frac{1}{V(t)}
    \geq
    \frac{1}{V(0)}
    +
    \int_0^t\frac{ds}{S(s)^2}.
\]
The first integral in Eq.~\ref{eq:app_divergent_integrals}
diverges, so $V(t)\to0$.
Consequently,
\begin{equation}
    \bar\Delta^*(t)\longrightarrow\Delta_K^*.
    \label{eq:app_weighted_mean_limit}
\end{equation}

\paragraph{3. The normalized weights concentrate on the
unique maximizer.}
For $i\neq K$, define
\[
    d_i:=\Delta_K^*-\Delta_i^*>0,
    \qquad
    d_{\min}:=\min_{i\neq K}d_i>0.
\]
Then
\[
    V(t)
    =
    \sum_{i\neq K}p_i(t)d_i
    \geq
    d_{\min}\bigl(1-p_K(t)\bigr).
\]
Since $V(t)\to0$, we have $p_K(t)\to1$.
For every $i\neq K$,
$0\leq p_i(t)\leq1-p_K(t)$, and hence $p_i(t)\to0$.

\paragraph{4. Every nonmaximal weight reaches the floor
in finite time.}
Fix $i\neq K$.
By Eq.~\ref{eq:app_weighted_mean_limit}, there exists
a finite $T_i$ such that
\[
    \bar\Delta^*(t)-\Delta_i^*
    \geq\frac{d_i}{2},
    \qquad t\geq T_i.
\]
Whenever $w_i(t)>\delta_0$ after $T_i$,
\begin{equation}
    \dot w_i(t)\leq-\frac{d_i}{2S(t)}.
    \label{eq:app_nonmaximal_weight_decrease}
\end{equation}
If $w_i$ never reached the floor after $T_i$, integration
would give
\[
    w_i(t)
    \leq
    w_i(T_i)
    -
    \frac{d_i}{2}\int_{T_i}^{t}\frac{ds}{S(s)}.
\]
The second integral in Eq.~\ref{eq:app_divergent_integrals}
diverges, so the right-hand side would eventually fall
below $\delta_0$, a contradiction.
Thus $w_i$ reaches the floor in finite time.

Once at the floor at or after $T_i$, its unprojected direction
remains negative, so the projection keeps it there.
Since there are finitely many nonmaximal coordinates,
there exists a finite $T_0$ such that
\begin{equation}
    w_i(t)=\delta_0
    \qquad
    \text{for all }i\neq K\text{ and }t\geq T_0.
    \label{eq:app_floor_hitting}
\end{equation}

For $\delta_0=0$, only $w_K(t)$ remains positive after $T_0$.
Therefore, $p_K(t)=1$ and $p_i(t)=0$ for every $i\neq K$
and $t\geq T_0$, proving finite-time convergence.

\paragraph{5. Convergence rate for a positive floor.}
Suppose $\delta_0>0$ and define
\[
    H:=(N-1)\delta_0>0,
    \qquad
    C:=\delta_0\sum_{i\neq K}d_i>0.
\]
For $t\geq T_0$, Eq.~\ref{eq:app_floor_hitting} gives
$S(t)=w_K(t)+H$ and
\[
    \Delta_K^*-\bar\Delta^*(t)
    =
    \frac{\delta_0\sum_{i\neq K}d_i}{S(t)}
    =
    \frac{C}{S(t)}.
\]
Only $w_K$ changes after $T_0$, so
\[
    \dot S(t)=\dot w_K(t)=\frac{C}{S(t)^2}.
\]
Multiplying by $3S(t)^2$ and integrating yields
\begin{equation}
    S(t)^3=S(T_0)^3+3C(t-T_0).
    \label{eq:app_normalization_exact_rate}
\end{equation}
Consequently,
\begin{align}
    1-p_K(t)
    &=
    \frac{H}{S(t)}
    \nonumber\\
    &=
    \frac{H}
    {\bigl[S(T_0)^3+3C(t-T_0)\bigr]^{1/3}}
    =
    O(t^{-1/3}).
\end{align}
This completes the proof.
\end{proof}

\section{Proof of Theorem~\ref{thm:feature_learning_comparison}}
\label{app:feature_learning_proof}
\paragraph{Notation and comparison condition.}
Denote the empirical covariance between random variables $u$ and $v$ as 
$\widehat{\operatorname{Cov}}_N(u,v)
:=N^{-1}\sum_{i=1}^{N}(u_i-\bar u)(v_i-\bar v).$
Further, denote each training data point as $\bm x_i=(1, e_i, h_i)^\top$. The features $e_i$ and $h_i$ are standardized to zero mean and unit standard deviation. 
The pseudo-labels
$\Delta_i=\Delta_i(\bm\theta_U^*,\bm\theta_W^*)$ are fixed.
For compactness, write $x_i^{(e)}=e_i$ and $x_i^{(h)}=h_i$, we have
\begin{equation}
\sum_i x_i^{(k)} = 0, \quad \frac{1}{N}\sum_i \left(x_i^{(k)} \right)^2 = 1, \quad k \in \{e, h\}.
\end{equation}
We call $e$ the easy feature and $h$ the hard feature because $e$ has stronger covariance with the target label than $h$, 
\begin{equation}
c_e=\widehat{\operatorname{Cov}}_N(\Delta,e),
\quad
c_h=\widehat{\operatorname{Cov}}_N(\Delta,h),
\quad |c_e|>|c_h|.
\end{equation}

As each feature is centered, for any vector $v\in\mathbb R^N$,
\begin{align}
\widehat{\operatorname{Cov}}_N(v,x^{(k)})
&=
\frac1N\sum_i(v_i-\bar v)x_i^{(k)}
\nonumber\\
&=
\frac1N\sum_i v_i x_i^{(k)}
-
\frac{\bar v}{N}\sum_i x_i^{(k)}
\nonumber\\
&=
\frac1N\sum_i v_i x_i^{(k)},
\qquad
\bar v=\frac1N\sum_i v_i.
\label{eq:app_centered_covariance}
\end{align}
Taking $v=\Delta^*$ gives
\begin{equation}
c_k
= \widehat{\operatorname{Cov}}_N(\Delta,x^{(k)}) =
\frac1N\sum_i\Delta_ix_i^{(k)},
\qquad
\sum_i\Delta_ix_i^{(k)}=Nc_k,
\quad k\in\{e,h\}.
\label{eq:app_target_covariance_identity}
\end{equation}

We have a meta-network parameterized by $\bm{\phi}_t=(\beta_{0,t},\beta_{e,t},\beta_{h,t})$. 
For each training data point $\bm x_i$ and training step $t$, the output $w_{i,t}$ of the meta-network is computed in the following way. 
\begin{equation}
a_{i,t}=\bm{\beta}_t^\top \bm{x}_i = \beta_{0,t}+\beta_{e,t} e_i+\beta_{h,t} h_i,
\qquad
w_{i,t}=s_{\bm{\phi}_t}(x_i)=a_{i,t}.
\end{equation}

Since the weights of data points are expected as non-negative, we assume that the initialization of parameters $\bm{\phi}_0$ satisfies $\beta_{e,0}=\beta_{h,0}=0$ and $\beta_{0,0}>0$ ensuring all $w_{i,0}>0$. Then there exists a training step $T>0$, such that $\forall t < T$, $w_{i,t}>0$ for all $i$. We only consider the training steps $t<T$.

Recall the two objectives:
\begin{align}
\mathcal L^{\mathrm{SBO}}(\bm{\phi}_t^{\mathrm{SBO}})
=
-\frac{\sum_iw_{i,t}^{\mathrm{SBO}}\Delta_i}{S_t},
\qquad
\mathcal L^{\mathrm{PVM}}(\bm{\phi}_t^{\mathrm{PVM}})
=
\frac1N\sum_i(w_{i,t}^{\mathrm{PVM}}-\Delta_i)^2,
\end{align}
where $S_t=\sum_iw_{i,t}^{\mathrm{SBO}}>0$ and
$\bar\Delta_t=\sum_iw_{i,t}^{\mathrm{SBO}}\Delta_i/S_t$.
We write
$G_k^q(\bm{\phi}_t^q)
=\left.\partial\mathcal L^q(\bm{\phi})/\partial\beta_k
\right|_{\bm{\phi}=\bm{\phi}_t^q}$
for $q\in\{\mathrm{SBO},\mathrm{PVM}\}$ and $k\in\{e,h\}$.

\begin{proof}
We first compare the feature gradients of ScaleBiO.
\paragraph{1. ScaleBiO shows a non-decreasing gap between easy and hard features.}
We first derive the ScaleBiO gradient with respect to parameters.
Since the pseudo-labels are fixed, the quotient rule gives
\begin{align}
G_k^{\mathrm{SBO}}(\bm{\phi}_t^{\mathrm{SBO}})
&=
-\frac1{S_t}
\sum_i\Delta_i
\frac{\partial w_{i,t}^{\mathrm{SBO}}}
     {\partial\beta_{k,t}^{\mathrm{SBO}}}
+
\frac{\sum_jw_{j,t}^{\mathrm{SBO}}\Delta_j^*}{S_t^2}
\sum_i
\frac{\partial w_{i,t}^{\mathrm{SBO}}}
     {\partial\beta_{k,t}^{\mathrm{SBO}}}
\nonumber\\
&=
-\frac1{S_t}
\sum_i
(\Delta_i-\bar\Delta_t)
\frac{\partial w_{i,t}^{\mathrm{SBO}}}
     {\partial\beta_{k,t}^{\mathrm{SBO}}}.
\label{eq:app_sbo_quotient_rule}
\end{align}
The meta-network definition implies
\begin{align}
\frac{\partial w_{i,t}^{\mathrm{SBO}}}
     {\partial\beta_{k,t}^{\mathrm{SBO}}}
=
x_i^{(k)}.
\end{align}
Substituting this into Eq.~\eqref{eq:app_sbo_quotient_rule} yields
\begin{align}
G_k^{\mathrm{SBO}}(\bm{\phi}_t^{\mathrm{SBO}})
&=
-\frac{1}{S_t}
\sum_i(\Delta_i-\bar\Delta_t)x_i^{(k)}
\nonumber\\
&=
-\frac{1}{S_t}
\left[
\sum_i\Delta_ix_i^{(k)}
-
\bar\Delta_t\sum_i x_i^{(k)}
\right]
\nonumber\\
&=
-\frac{1}{S_t}
\sum_i\Delta_ix_i^{(k)}
\nonumber\\
&=
-\frac{N}{S_t}c_k.
\label{eq:app_sbo_covariance_gradient}
\end{align}
The third equality uses $\sum_i x_i^{(k)}=0$.
The final equality uses
$\sum_i\Delta_ix_i^{(k)}=Nc_k$
from Eq.~\ref{eq:app_target_covariance_identity}.
Since the positive factor $N/S_t$ is shared by both features,
\begin{equation}
(G_e^{\mathrm{SBO}}(\bm{\phi}_t^{\mathrm{SBO}}))^2
-
(G_h^{\mathrm{SBO}}(\bm{\phi}_t^{\mathrm{SBO}}))^2
=
\left(\frac{N}{S_t}\right)^2(c_e^2-c_h^2)
>0.
\label{eq:app_sbo_gradient_ordering}
\end{equation}

Next, we show that \(S_t\) is non-increasing,
\begin{align}
S_t
&=
\sum_i w_{i,t}^{\mathrm{SBO}}
\nonumber\\
&=
N \beta_{0,t}^{\mathrm{SBO}}
+\beta_{e,t}^{\mathrm{SBO}}\sum_i e_i
+\beta_{h,t}^{\mathrm{SBO}}\sum_i h_i
\nonumber\\
&=N \beta_{0,t}^{\mathrm{SBO}}.
\label{eq:app_sbo_weight_sum}
\end{align}

Define
\begin{equation}
z_t
=
c_e\beta_{e,t}^{\mathrm{SBO}}
+c_h\beta_{h,t}^{\mathrm{SBO}}.
\end{equation}
Using Eq.~\eqref{eq:app_sbo_covariance_gradient},
the gradient descent updates give
\begin{equation}
\beta_{k,t+1}^{\mathrm{SBO}}
=
\beta_{k,t}^{\mathrm{SBO}}
+\eta\frac{ N}{S_t}c_k,
\qquad k\in\{e,h\}.
\end{equation}
Consequently,
\begin{equation}
z_{t+1}
=
z_t+\eta\frac{ N}{S_t}(c_e^2+c_h^2).
\label{eq:app_sbo_correlation_update}
\end{equation}
Since $z_0=0$ and $S_t>0$, it follows that $z_t\ge0$.

To determine the bias update, write
$\Delta_{\mathrm{avg}}=N^{-1}\sum_i\Delta_i$.
Eq.~\ref{eq:app_sbo_weight_sum} and
Eq.~\ref{eq:app_target_covariance_identity} imply
\begin{align}
\sum_i w_{i,t}^{\mathrm{SBO}}\Delta_i
&=
\beta_{0,t}^{\mathrm{SBO}}
\sum_i\Delta_i
+ N z_t
\nonumber\\
&=
S_t\Delta_{\mathrm{avg}}+ N z_t.
\end{align}
Therefore,
\begin{align}
\mathcal L^{\mathrm{SBO}}(\bm{\phi}_t^{\mathrm{SBO}})
=
-\Delta_{\mathrm{avg}}
-\frac{ N z_t}{S_t}.
 \end{align}
Since $z_t$ does not depend on the bias coordinate and
$\partial S_t/\partial b_t^{\mathrm{SBO}}= N$,
the bias gradient is
\begin{equation}
\frac{d \mathcal{L}^{\mathrm{SBO}}}{d\beta_{0,t}}
=
\frac{N^2}{S_t^2}z_t
\ge0.
\label{eq:app_sbo_bias_gradient}
\end{equation}
Thus the bias update satisfies
 \begin{align}
\beta_{0,t+1}^{\mathrm{SBO}}
=
\beta_{0,t}^{\mathrm{SBO}}
-\eta\frac{N^2}{S_t^2}z_t
\le \beta_{0,t}^{\mathrm{SBO}}.
 \end{align}
Combining this with
Eq.~\eqref{eq:app_sbo_weight_sum} yields
\begin{equation}
S_{t+1}
=
S_t-\eta\frac{3N^3}{S_t^2}z_t
\le S_t.
\label{eq:app_sbo_weight_sum_monotonicity}
\end{equation}

Finally, Eq.~\eqref{eq:app_sbo_gradient_ordering} gives
\begin{align}
D_t^{\mathrm{SBO}}
&=
\left(\frac{ N}{S_t}\right)^2
(c_e^2-c_h^2)
\nonumber\\
&=
\left(\frac{ N}{S_{t-1}}\right)^2
(c_e^2-c_h^2)
\nonumber\\
&\ge D_{t-1}^{\mathrm{SBO}}.
\label{eq:app_sbo_gradient_gap_lower_bound}
\end{align}

Thus, ScaleBiO continues to learn easy features than hard features, and their gradient gap does not decrease. Next we show the exponentially decreasing of gradient gap with PVM.

\paragraph{2. PVM exhibits a exponentially decreasing gradient gap.}
We now analyze the PVM objective under the same fixed affine surrogate.
Define the residual
 \begin{align}
r_{i,t}
=
w_{i,t}^{\mathrm{PVM}}-\Delta_i.
 \end{align}
Since the pseudo-labels are fixed, the feature gradients are
\begin{align}
G_k^{\mathrm{PVM}}(\bm{\phi}_t^{\mathrm{PVM}})
&=
\frac2N\sum_i
r_{i,t}
\frac{\partial w_{i,t}^{\mathrm{PVM}}}
     {\partial\beta_{k,t}^{\mathrm{PVM}}}
\nonumber\\
&=
\frac{2}{N}\sum_i r_{i,t}x_i^{(k)},
\qquad k\in\{e,h\}.
\label{eq:app_pvm_feature_gradient}
\end{align}

To include the bias update, write
The full gradient is therefore
\begin{equation}
x_i=(1,e_i,h_i)^\top,
\qquad
g_t
=
\nabla_{\bm{\phi}}\mathcal L^{\mathrm{PVM}}
(\bm{\phi}_t^{\mathrm{PVM}}).
\end{equation}
The full gradient is therefore
\begin{equation}
g_t=\frac{2}{N}\sum_i r_{i,t}x_i.
\label{eq:app_pvm_full_gradient}
\end{equation}
Each network output is,
\begin{align}
w_{i,t}^{\mathrm{PVM}}
=x_i^\top\bm{\phi}_t^{\mathrm{PVM}}.
\end{align}
Combining this with the gradient descent update
$\bm{\phi}_{t+1}^{\mathrm{PVM}}
=\bm{\phi}_t^{\mathrm{PVM}}-\eta g_t$ gives
\begin{align}
r_{i,t+1}
&=
r_{i,t}
+ x_i^\top
\left(\bm{\phi}_{t+1}^{\mathrm{PVM}}
-\bm{\bm{\phi}}_t^{\mathrm{PVM}}\right)
\nonumber\\
&=
r_{i,t}-\eta x_i^\top g_t.
\label{eq:app_pvm_residual_update}
\end{align}
Substituting Eq.~\eqref{eq:app_pvm_residual_update}
into Eq.~\eqref{eq:app_pvm_full_gradient} yields
\begin{align}
g_{t+1}
&=
\frac{2}{N}
\sum_i
\left(r_{i,t}-\eta x_i^\top g_t\right)x_i
\nonumber\\
&=
g_t
-
2\eta
\left(\frac1N\sum_i x_ix_i^\top\right)g_t
\nonumber\\
&=
(I-2\eta Q)g_t,
\label{eq:app_pvm_gradient_update}
\end{align}
where
 \begin{align}
Q=\frac1N\sum_i x_ix_i^\top.
 \end{align}
Because the features are centered and have unit variance,
\begin{equation}
Q=
\begin{pmatrix}
1&0&0\\
0&1&\rho\\
0&\rho&1
\end{pmatrix},
\qquad
\rho=\frac1N\sum_i e_i h_i.
\label{eq:app_feature_second_moment}
\end{equation}

For brevity, write
$G_{k,t}^{\mathrm{PVM}}
=G_k^{\mathrm{PVM}}(\bm{\phi}_t^{\mathrm{PVM}})$.
Equations~\eqref{eq:app_pvm_gradient_update}
and~\eqref{eq:app_feature_second_moment} imply
\begin{align}
G_{e,t+1}^{\mathrm{PVM}}
&=
(1-2\eta)G_{e,t}^{\mathrm{PVM}}
-2\eta\rho G_{h,t}^{\mathrm{PVM}},
\nonumber\\
G_{h,t+1}^{\mathrm{PVM}}
&=
(1-2\eta)G_{h,t}^{\mathrm{PVM}}
-2\eta\rho G_{e,t}^{\mathrm{PVM}}.
\label{eq:app_pvm_feature_gradient_update}
\end{align}
Consequently, the squared-gradient gap satisfies
\begin{align}
D_{t+1}^{\mathrm{PVM}}
&=
\left[
(1-2\eta)G_{e,t}^{\mathrm{PVM}}
-2\eta\rho G_{h,t}^{\mathrm{PVM}}
\right]^2
\nonumber\\
&\quad-
\left[
(1-2\eta)G_{h,t}^{\mathrm{PVM}}
-2\eta\rho G_{e,t}^{\mathrm{PVM}}
\right]^2
\nonumber\\
&=
\left[(1-2\eta)^2-4\eta^2\rho^2\right]
\left[
(G_{e,t}^{\mathrm{PVM}})^2
-(G_{h,t}^{\mathrm{PVM}})^2
\right]
\nonumber\\
&=
\left[(1-2\eta)^2-4\eta^2\rho^2\right]
D_t^{\mathrm{PVM}}.
\label{eq:app_pvm_gradient_gap_recursion}
\end{align}
The cross terms cancel in the third equality.

Choose the step size such that:
 \begin{align}
0<\eta<\frac1{2+2|\rho|}.
 \end{align}
Then
 \begin{align}
0<
(1-2\eta)^2-4\eta^2\rho^2
=
[1-2\eta(1+\rho)]
[1-2\eta(1-\rho)]
<1.
 \end{align}


Define
\begin{equation}
\lambda
=
1-\left[(1-2\eta)^2-4\eta^2\rho^2\right]
=
4\eta-(1-\rho^2)4\eta^2, \qquad \lambda\in(0,1)
\label{eq:app_pvm_contraction_parameter}
\end{equation}
Equation~\eqref{eq:app_pvm_gradient_gap_recursion}
therefore becomes
 \begin{align}
D_{t+1}^{\mathrm{PVM}}
=
(1-\lambda)D_t^{\mathrm{PVM}}.
 \end{align}
Iterating this identity gives
\begin{equation}
D_t^{\mathrm{PVM}}
=
(1-\lambda)^tD_0^{\mathrm{PVM}}.
\label{eq:app_pvm_gradient_gap_decay}
\end{equation}

Thus PVM exponentially contracts the absolute squared-gradient
gap, while Eq.~\eqref{eq:app_sbo_gradient_gap_lower_bound}
shows that the ScaleBiO gap remains no lower than its initial
value throughout the stated comparison interval.
\end{proof}

\section{Experiment Details}
\label{app:experiment_details}

\subsection{Datasets}
\label{app:dataset_details}

For LLM safety, we use Alpaca and Dolly as training datasets and evaluate on DirectHarm4, HarmBench, and HEx-PHI:
\begin{itemize}
    \item Alpaca~\citep{alpaca,peng2023instruction} contains
    52,002 instruction--response pairs covering diverse
    instruction-following tasks.

    \item Dolly~\citep{DatabricksBlog2023DollyV2} contains
    15,011 human-written instruction--response pairs covering
    tasks such as question answering, summarization, and
    creative writing.

    \item DirectHarm4 (DH4)~\citep{lyu2024keeping} contains
    400 harmful instructions, with 100 examples in each of
    four categories: malware, drugs, phishing, and disinformation.

    \item HarmBench (HB)~\citep{mazeika2024harmbench} evaluates
    harmful behavior generation.
    Its text-only portion contains 400 behaviors across
    standard, contextual, and copyright-related categories.

    \item HEx-PHI (HEx)~\citep{qi2024fine} evaluates responses
    to instructions spanning prohibited-use categories.
    Its original release contains 330 instructions, with
    30 examples in each of 11 categories: illegal activity,
    child abuse content, hate/harassment/violence, malware,
    physical harm, economic harm, fraud/deception, adult content,
    political campaigning, privacy violation, and tailored
    financial advice.
\end{itemize}

For targeted instruction tuning, we use Tulu V2 as the candidate instruction corpus and evaluate data selection for GSM8K and CodeX:
\begin{itemize}
    \item Tulu V2~\citep{ivison2025large} is a multi-source
    instruction-tuning corpus; we use the approximately
    200K-example version preprocessed by
    \citet{nayak2026critical}, which excludes examples with
    no response tokens within the first 2,048 tokens.
    We train the meta-network on a random 50K subset and apply
    it to the full candidate pool.

    \item GSM8K~\citep{cobbe2021training} contains
    grade-school math problems with step-by-step solutions,
    comprising 7,473 training examples and 1,319 test examples.
    We hold out 1,000 examples from the training split as
    a shared validation set for all methods.

    \item CodeX~\citep{chen2021evaluating} contains
    164 Python function-completion problems with unit tests
    for evaluating functional correctness.
    Following \citet{ivison2025large}, we use 16 examples
    as the query set and the remaining 148 as the test set.
\end{itemize}

\subsection{Implementation Details}
\label{app:instruction_exp_details}

\textbf{meta-network learning.}
We train the two LLMs, $\bm\theta_U$ and $\bm\theta_W$, using
LISA~\citep{pan2024lisa}, with learning rates in
$\{5\times10^{-5},10^{-4}\}$ for up to three epochs.
We tune the training--validation trade-off coefficient $\alpha$
over $\{0.01,0.1,1.0\}$.
The meta-network is a three-layer MLP with a sigmoid output,
optimized using AdamW with a learning rate of $10^{-3}$.

\textbf{Downstream fine-tuning for LLM safety.}
We apply LoRA to the query and key projections, $W_q$ and $W_k$,
in every Transformer layer, with rank $r=8$ and
$\alpha_{\mathrm{LoRA}}=8$.
We train for three epochs with a batch size of 32,
a learning rate of $10^{-4}$, and a cosine learning rate schedule
with a warmup ratio of 0.1.

\textbf{Downstream fine-tuning for targeted instruction tuning.}
We fully finetune Llama-2-7B on the selected examples for two
epochs using AdamW with a learning rate of $2\times10^{-5}$
and a linear learning rate schedule.
The maximum sequence length is 2,048 tokens.

\subsection{Input Features to the Meta-Network}
\label{app:selector_features}

Inspired by \citet{du2026difficulty}, we construct sample-level
features in both representation and gradient spaces.
The representation embedding is the final-layer hidden state
of the last token, and the gradient embedding is derived from
the per-sample loss gradient.
In each space, we compute the distances to each of the three
nearest validation examples and to the mean validation embedding.
We concatenate these eight distances with the sample's loss
and gradient norm to obtain a 10-dimensional feature vector.

For LLM safety, we append a binary feature indicating whether
the sample contains numbers or predefined keywords such as
\texttt{list}, \texttt{convert}, and \texttt{calculate}.
This feature is motivated by findings that some benign
list-format and mathematical examples can degrade safety
alignment after fine-tuning~\citep{chen2024your}.

\subsection{Evaluation Details}
\label{app:evaluation_details}

\textbf{LLM safety.}
For DH4, HB, and HEx, we generate model responses using greedy
decoding (temperature 0).
We classify each response as safe or unsafe using Llama Guard,
also with greedy decoding.
We report the attack success rate (ASR), defined as the percentage
of evaluated responses classified as unsafe; lower is better.

\textbf{GSM8K.}
Following \citet{ivison2025large}, we use the eight
chain-of-thought demonstrations as individual query examples
for data selection, without prepending other demonstrations
to each query.
At evaluation time, we include all eight demonstrations in
each test prompt and report exact-match accuracy of the
extracted final numerical answers on the 1,319 test problems.

\textbf{CodeX.}
We sample code completions at a temperature of 0.8 and report
pass@10 on the 148 held-out problems.
This metric estimates the probability that at least one of
ten sampled completions passes all unit tests for a problem.


\subsection{KL Loss in Ablation Study}
\label{app:KL_loss}
The KL variant matches the normalized outputs of the meta-network $\mathbf p$ to a
distribution $\mathbf q$ reweighted by the loss-gap pseudo-labels:
$$
    p_i=\frac{s_{\bm\phi}(x_i)}{\sum_j s_{\bm\phi}(x_j)},
    \qquad
    q_i=\frac{p_i\exp(\Delta_i^*)}
    {\sum_j p_j\exp(\Delta_j^*)},
    \qquad
    \mathcal L_{\bm\phi}^{\mathrm{KL}}
    =\operatorname{KL}(\mathbf q\,\|\,\mathbf p).
$$

\section{Toy Experiment in Sec.\ref{sec:selector_objective}}
\label{app:toy_exp_details}

We construct disjoint seen and unseen datasets, each containing
2,575 safe examples from GSM8K~\citep{cobbe2021training} and
1,125 unsafe examples from BeaverTails~\citep{ji2023beavertails}.
The unsafe examples pair harmful prompts with harmful responses.
For validation, we use 100 safety-related examples
from \citet{he2024your}, each pairing a harmful prompt with
a refusal response.

We train the meta-network on the seen set for 30 epochs using AdamW
with a batch size of 256, an initial learning rate of
$5\times10^{-4}$, and a cosine learning rate schedule
that decays to zero.

For each split separately, we rank examples by predicted weight
and classify the $k=1{,}125$ lowest-weight examples as unsafe.
Let $\mathcal B_k$ denote this predicted set and $\mathcal U$
the ground-truth unsafe set.
We define unsafe-identification accuracy as
\begin{equation}
    \mathrm{Acc}
    =
    \frac{|\mathcal B_k\cap\mathcal U|}{k}\times100\%.
\end{equation}
Since $|\mathcal U|=k$, this metric equals both precision
and recall for unsafe-example identification.

\textbf{Shortcut learning.}
To test shortcut reliance, we append three synthetic features
to each example before standardization.
Let $c=+1$ for GSM8K examples and $c=-1$ for unsafe examples,
and define $\mathbf v=(1,-0.5,0.2)^\top$.
We generate the shortcut features as
\begin{equation}
    \mathbf z_{\mathrm{seen}}
    =ac\mathbf v+\boldsymbol\epsilon,
    \qquad
    \mathbf z_{\mathrm{unseen}}
    =-ac\mathbf v+\boldsymbol\epsilon',
\end{equation}
where $a$ controls shortcut strength and
$\boldsymbol\epsilon,\boldsymbol\epsilon'
\sim\mathcal N(\mathbf 0,\sigma^2\mathbf I_3)$,
with noise standard deviation $\sigma=0.3$.
We use $a=0.5$ and $a=0.3$ for the strong- and weak-shortcut
settings, respectively.
Noise is sampled independently across examples and coordinates.
This construction reverses the class-dependent shortcut means
on unseen data.

For each split, we summarize the associations between the
shortcut features and class labels before standardization.
For feature $j\in\{1,2,3\}$, let
\begin{equation}
    \widehat C_j
    =
    \frac{1}{n-1}
    \sum_{i=1}^{n}
    (z_{ij}-\bar z_j)(c_i-\bar c)
\end{equation}
denote its empirical covariance with the label, where $n$ is
the split size and $\bar z_j$ and $\bar c$ are the sample means.
We report the mean absolute covariance and mean absolute
Pearson correlation:
\begin{equation}
    \mathrm{Cov}
    =
    \frac{1}{3}\sum_{j=1}^{3}|\widehat C_j|,
    \qquad
    \rho
    =
    \frac{1}{3}\sum_{j=1}^{3}
    \left|
        \frac{\widehat C_j}
        {\widehat\sigma_j\widehat\sigma_c}
    \right|,
\end{equation}
where $\widehat\sigma_j$ and $\widehat\sigma_c$ are the
sample standard deviations of feature $j$ and the label,
respectively.

We report seen accuracy $\mathrm{Acc}_{\mathrm S}$,
unseen accuracy $\mathrm{Acc}_{\mathrm U}$, and their gap
$G=\mathrm{Acc}_{\mathrm S}-\mathrm{Acc}_{\mathrm U}$
in percentage points.
At comparable seen accuracy, a smaller gap indicates
better generalization under shortcut reversal.

\section{Small-to-Large Model Generalization}
\label{app:small2big}

Table~\ref{tab:safety_transfer} provides the detailed results of the
small-to-large model transfer experiment summarized in
Sec.~\ref{sec:safety_results}. Small2Big achieves a
$5.65$--$6.60\times$ training speedup while retaining $77.53\%$ and
$83.67\%$ of Big2Big's average ASR in the \textit{Training} and
\textit{Generalization} settings, respectively. These results show that a meta-network trained with \shortname using a small LLM
can generalize to data selection for a larger model. 
performance.

\begin{table*}[t]
\centering
\small
\setlength{\tabcolsep}{3.5pt}
\caption{
Small-to-large model transfer with \shortname.
Small2Big uses Qwen2.5-0.5B-Instruct for meta-network training and
Qwen2.5-7B-Instruct for fine-tuning on the selected data; Big2Big uses
Qwen2.5-7B-Instruct for both.
We report ASR (\%) and training speedup $S_{\mathrm{rel}}$ relative to
Big2Big within each dataset and setting.
Avg. averages the six ASR scores, and Retention (\%) is the ratio of
Small2Big's average ASR to Big2Big's within the same setting.
Higher is better for all metrics; better ASR and speedup results are
\textbf{bold}.
}
\label{tab:safety_transfer}

\begin{tabular}{@{}llcccccccccc@{}}
\toprule
\multirow{2}{*}{Setting}
& \multirow{2}{*}{Method}
& \multicolumn{4}{c}{Alpaca}
& \multicolumn{4}{c}{Dolly}
& \multicolumn{2}{c}{Overall ASR} \\
\cmidrule(lr){3-6}
\cmidrule(lr){7-10}
\cmidrule(lr){11-12}
& & DH4 & HB & HEx & $S_{\mathrm{rel}}$
& DH4 & HB & HEx & $S_{\mathrm{rel}}$
& Avg. & Retention \\
\midrule

\multirow{2}{*}{Training}
& Big2Big
& \textbf{44.50}
& \textbf{23.50}
& \textbf{24.83}
& 1.00$\times$
& \textbf{86.50}
& \textbf{87.00}
& \textbf{88.62}
& 1.00$\times$
& \textbf{59.16}
& -- \\

& Small2Big
& 35.50
& 13.00
& 13.79
& \textbf{6.60}$\times$
& 74.00
& 72.00
& 66.90
& \textbf{5.65}$\times$
& 45.87
& 77.53 \\

\midrule

\multirow{2}{*}{Generalization}
& Big2Big
& \textbf{38.75}
& \textbf{16.00}
& \textbf{18.28}
& 1.00$\times$
& \textbf{84.50}
& \textbf{83.50}
& \textbf{80.00}
& 1.00$\times$
& \textbf{53.51}
& -- \\

& Small2Big
& 32.25
& 13.00
& 14.14
& \textbf{6.20}$\times$
& 77.50
& 70.00
& 61.72
& \textbf{5.65}$\times$
& 44.77
& 83.67 \\

\bottomrule
\end{tabular}
\end{table*}

\section{Conclusion} 

In this paper, we introduced \longname (\shortname),  transferable data-selection framework that utilize sample-level pseudo-labels to learn a data selection network generalizable to unseen data.
Our theoretical and empirical analyses identify competitive
weight suppression and persistent easy-feature preference
as limitations of the ScaleBiO meta-objective.
To address these issues, we propose a new training objective PVM.
Experiments demonstrate cross-dataset, subset-to-full-corpus,
and small-to-large model transfer.
Specifically, in terms of generalization, \shortname outperforms the strongest learnable-weighting baselines by $22.75$ and $3.99$ points on LLM safety and targeted instruction tuning, respectively.











\end{document}